\documentclass{article}
\usepackage[eatrack]{log_2022}

\usepackage{booktabs}           
\usepackage{multirow}           
\usepackage{amsfonts}           
\usepackage{graphicx}           
\usepackage{duckuments}         
\usepackage{wrapfig}
\usepackage{amssymb}
\usepackage{amsthm}
\usepackage{relsize}
\usepackage{rotating}

\newcommand*{\ldblbrace}{\{\mskip-5mu\{}
\newcommand*{\rdblbrace}{\}\mskip-5mu\}}
\newcommand*{\bigldblbrace}{\bigg\{\mskip-10mu\bigg\{}
\newcommand*{\bigrdblbrace}{\bigg\}\mskip-10mu\bigg\}}

\newtheorem{lemma}{Lemma}
\newtheorem{theorem}{Theorem}
\newtheorem{remark}{Remark}
\newtheorem{definition}{Definition}

\usepackage{xargs}
\usepackage[prependcaption,textsize=tiny]{todonotes}

\def\finalSubmission{1}
\def\includeAcknowledgements{1}

\def\useNotes{1}

\if\finalSubmission1
  \def\useNotes{0}
\fi

\ifx\useNotes\undefined
\else
  \if\useNotes1
    \newcommandx{\VGn}[2][1=]{\todo[linecolor=blue,backgroundcolor=blue!25,bordercolor=blue,#1]{#2}}
    \newcommandx{\NAn}[2][1=]{\todo[linecolor=red,backgroundcolor=red!25,bordercolor=red,#1]{#2}}
    \newcommandx{\AKn}[2][1=]{\todo[linecolor=green,backgroundcolor=green!25,bordercolor=green,#1]{#2}}
    \newcommandx{\DNn}[2][1=]{\todo[linecolor=gray,backgroundcolor=gray!25,bordercolor=gray,#1]{OK: #2}}
    \newcommandx{\VGnOk}[2][1=]{\todo[linecolor=blue,backgroundcolor=blue!10,bordercolor=blue,#1]{#2}}
    \newcommandx{\NAnOk}[2][1=]{\todo[linecolor=red,backgroundcolor=red!10,bordercolor=red,#1]{#2}}
    \newcommandx{\AKnOk}[2][1=]{\todo[linecolor=green,backgroundcolor=green!10,bordercolor=green,#1]{#2}}
  \else
    \newcommandx{\VGn}[2][1=]{}
    \newcommandx{\NAn}[2][1=]{}
    \newcommandx{\AKn}[2][1=]{}
    \newcommandx{\DNn}[2][1=]{}
    \newcommandx{\VGnOk}[2][1=]{}
    \newcommandx{\NAnOk}[2][1=]{}
    \newcommandx{\AKnOk}[2][1=]{}
  \fi
\fi

\usepackage{algorithm}
\usepackage{algorithmicx}
\usepackage{algpseudocode}
\algdef{SE}[DOWHILE]{Do}{DoWhile}{\algorithmicdo}[1]{\algorithmicwhile\ #1}

\algrenewcommand\algorithmicrequire{\textbf{Input:}}
\algrenewcommand\algorithmicensure{\textbf{Output:}}

\usepackage[numbers,compress,sort]{natbib}

\title[Beyond 1-WL with Local Ego-Network Encodings]{Beyond 1-WL with Local Ego-Network Encodings}

\author[N. Alvarez-Gonzalez et al.]{%
Nurudin Alvarez-Gonzalez\thanks{Corresponding Author.}\\
\institute{Universitat Pompeu Fabra}\\
\email{nuralgon@gmail.com}\And
Andreas Kaltenbrunner\\
\institute{ISI Foundation}\\
\institute{Universitat Pompeu Fabra}\\
\email{kaltenbrunner@gmail.com}\And
Vicen\c{c} G\'omez\\
\institute{Universitat Pompeu Fabra}\\
\email{vicen.gomez@upf.edu}
}

\begin{document}

\maketitle

\begin{abstract}
Identifying similar network structures is key to capture graph isomorphisms and learn representations that exploit structural information encoded in graph data. This work shows that ego-networks can produce a structural encoding scheme for arbitrary graphs with greater expressivity than the Weisfeiler-Lehman (1-WL) test. We introduce $\textsc{Igel}$, a preprocessing step to produce features that augment node representations by encoding ego-networks into sparse vectors that enrich Message Passing (MP) Graph Neural Networks (GNNs) beyond 1-WL expressivity. We describe formally the relation between $\textsc{Igel}$ and 1-WL, and characterize its expressive power and limitations. Experiments show that $\textsc{Igel}$ matches the empirical expressivity of state-of-the-art methods on isomorphism detection while improving performance on seven GNN architectures.
\end{abstract}

\section{Introduction}
Novel approaches for representation learning on graph structured data have appeared in recent years~\citep{Bronstein2021GeometricDL}.
Graph neural networks can efficiently learn representations that depend both on the graph structure and node and edge features from large-scale graph datasets.
The most popular choice of architecture is the Message Passing Graph Neural Network (MP-GNN). In MP-GNNs, a node is represented by iteratively aggregating local feature `messages' from its neighbors.
Despite being succesfully applied in a wide variety of domains~\citep{Ying-2018-PinSAGE,DuvenaudNIPS2015,gilmer2017neural,JMLR:v21:19-671,battaglia_NIPS2016}, there is a limit on the representational power of MP-GNNs provided by the computationally efficient Weisfeiler-Lehman (1-WL) test for checking graph isomorphism~\citep{xu2018how,morris2019}.
Establishing this connection has lead to a better theoretical understanding of the performance of MP-GNNs and many possible generalizations 
\citep{grohe_2017,Brijder2019,Barcelo2020,geerts2021,Morris21a}.

To improve the expressivity of MP-GNNs, recent methods have extended the vanilla message-passing mechanism is various ways. For example,
using higher order $k$-vertex tuples~\citep{morris2019} leading to $k$-WL generalizations, introducing relative positioning information for network vertices~\citep{pmlr-v97-you19b}, propagating messages beyond direct neighborhoods~\citep{nikolentzos2020khop}, using concepts from algebraic topology~\citep{bodnar2021}, or combining sub-graph information in different ways~\citep{HaggaiNIPS2019PPGN,bouritsas2021,zhang2021,balcilar2021breaking,you2021identity,sandfelder2021ego,zhao2022,bevilacqua2022equivariant,DEA_GNNs}. 
All aforementioned approaches (which we review in more detail in~\autoref{sec:relwork}) improve expressivity by extending MP-GNNs architectures, often evaluating on standarized benchmarks~\citep{hu2020ogb,Morris2020TUDatasets,you2020design,dwivedi2020benchmarkgnns}.
However, identifying the optimal approach on novel domains remains unclear and requires costly architecture search.


In this work, we show that incorporating simple ego-network encodings already boosts the expressive power of MP-GNNs beyond the 1-WL test, while keeping the benefits of efficiency and simplicity.
We present \textbf{\textsc{Igel}}, an \textbf{I}nductive \textbf{G}raph \textbf{E}ncoding of \textbf{L}ocal information, which in its basic form extends node attributes with histograms of node degrees at different distances.
The \textbf{\textsc{Igel}} encodings can be computed as a pre-processing step irrespective of model architecture. 
Theoretically, we formally prove that the \textbf{\textsc{Igel}} encoding is no less expressive than the 1-WL test, and provide examples that show that it is more expressive than 1-WL. We also identify expressivity upper-bounds for graphs that are indistinguishable using state of the art methods.
Experimentally, we asses the performance of seven model architectures enriched with \textbf{\textsc{Igel}} encodings on five tasks and ten graph data sets, and find that it consistently improves downstream model performance.


\section{\texorpdfstring{$\textsc{Igel}$}{IGEL}: Ego-Networks As Sparse Inductive Representations}\label{sec:IGELAlgorithm}

Given a graph $G = (V, E)$, we define $n =  \vert V \vert $ and $m =  \vert E \vert $, $d_G(v)$ is the degree of a node $v$ in $G$ and $d_{\texttt{max}}$ is the maximum degree. For $u, v \in V$, $l_G(u, v)$ is their shortest distance, and $\texttt{diam}(G) = \max(l_G(u, v) \vert u, v \in V)$ is the diameter of $G$. Double brackets $\ldblbrace\cdot\rdblbrace$ denote a lexicographically-ordered     multi-set, $\mathcal{E}^{\alpha}_v \subseteq G$ is the $\alpha$-depth ego-network centered on $v$, and $\mathcal{N}_G^{\alpha}(v)$ is the set of neighbors of $v$ in $G$ up to distance $\alpha$, i.e., $\mathcal{N}_G^{\alpha}(v) = \{u\ |\ u\ \in V\  \land \ l_G(u, v) \leq \alpha\}$.

\autoref{alg:ColorRefinement1WL} shows the 1-WL test, where $\texttt{hash}$ maps a multi-set to an equivalence class shared by all nodes with matching multi-set encodings after a 1-WL iteration. The output of 1-WL is $\mathbb{N}^n$---mapping each node to a color, bounded by $n$ distinct colors if each node is uniquely colored. $k$-higher order variants of the WL test (denoted $k$-WL) operate on $k$-tuples of vertices, such that colors are assigned to $k$-vertex tuples. If two graphs $G_1$, $G_2$ are not distinguishable by the $k$-WL test (that is, their coloring histograms match), they are $k$-WL equivalent---denoted $G_1 \equiv_{k-\text{WL}} G_2$. Due to the hashing step, 1-WL does not preserve distance information in the encoding, and minor changes in the structure of the network (removing one edge) may dramatically change node-level representations. $\textsc{Igel}$ addresses both limitations, improving expressivity in the process.

\subsection{The \texorpdfstring{$\textsc{Igel}$}{IGEL} Algorithm}

Intuitively, $\textsc{Igel}$ encodes a vertex $v$ with the multi-set of ordered degree sequences at each distance within $\mathcal{E}^{\alpha}_v$. As such, $\textsc{Igel}$ is a variant of the 1-WL algorithm shown in~\autoref{alg:ColorRefinement1WL}, executed for $\alpha$ steps with two modifications. First, the hashing step is removed and replaced by computing the union of multi-sets across steps ($\cup$); second, the iteration number is explicitly introduced in the representation---with the output multi-set $e^{\alpha}_v$ shown in~\autoref{alg:IGELEncoding}.

\begin{wrapfigure}{r}{0.49\linewidth}
    \vspace{-.5cm}
    \centering
    \includegraphics[width=0.9\linewidth]{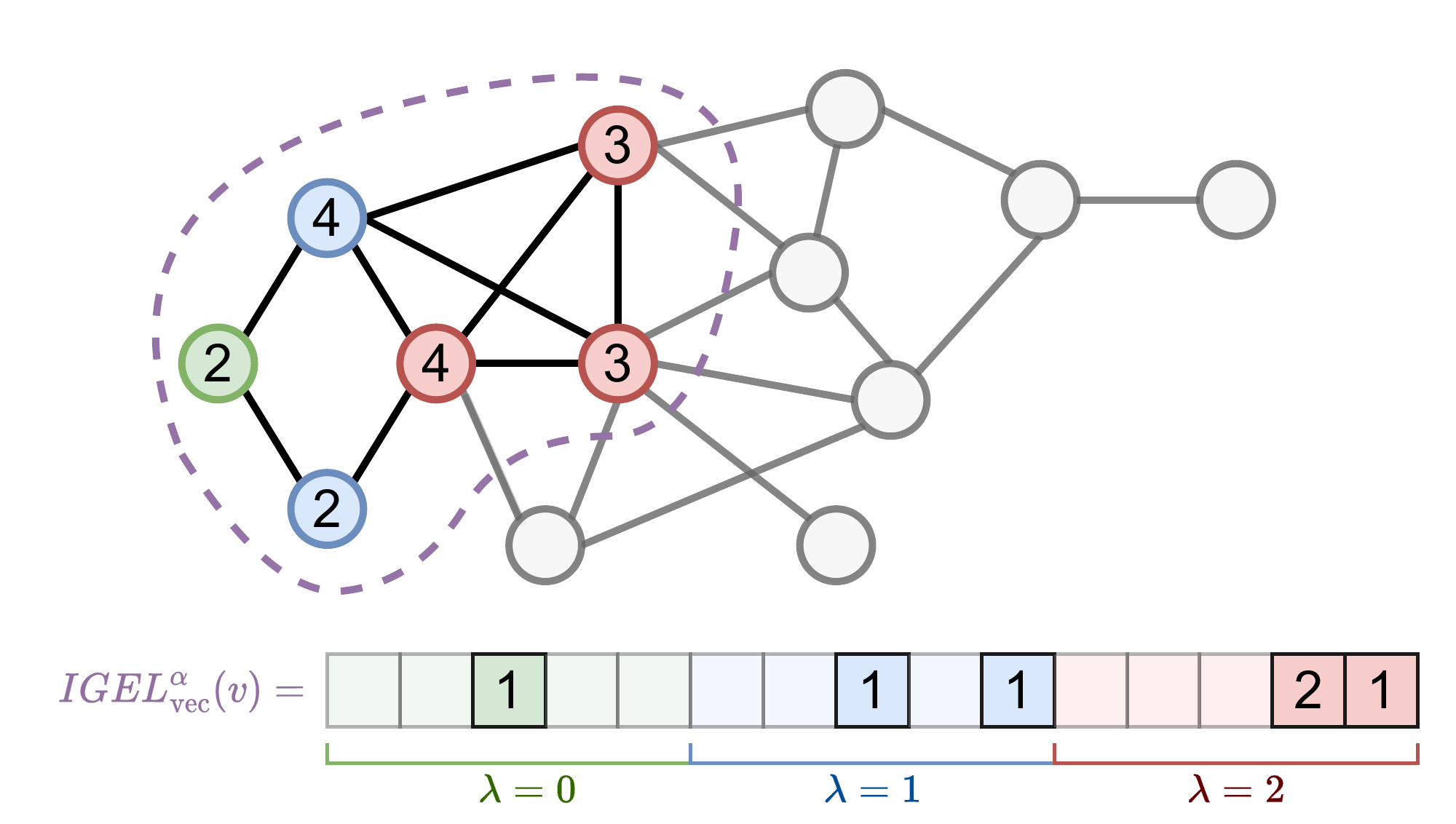}
    \caption{$\textsc{Igel}$ encoding of the \textcolor{green!50!black}{green vertex}. Dashed region denotes $\mathcal{E}^{\alpha}_v (\alpha = 2)$. The \textcolor{green!50!black}{green} vertex is at distance \textcolor{green!50!black}{$0$}, \textcolor{blue!75!black}{blue vertices} at \textcolor{blue!75!black}{$1$} and \textcolor{red!80!black}{red vertices} at \textcolor{red!80!black}{$2$}. Labels show degrees in $\mathcal{E}^{\alpha}_v$. The frequency of $(\lambda, \delta)$ tuples forming $\textsc{Igel}^{\alpha}_{\texttt{vec}}(v)$ is: {\tiny$\{
        \textcolor{green!50!black}{(0, 2): 1},
        \textcolor{blue!75!black}{(1, 2): 1},
        \textcolor{blue!75!black}{(1, 4): 1},
        \textcolor{red!80!black}{(2, 3): 2},
        \textcolor{red!80!black}{(2, 4): 1} 
    \}$}.
    }    
    \label{fig:IGELGraphEncoding}
\end{wrapfigure}
To be used as vertex features, the multi-set can be represented as a sparse vector $\textsc{Igel}_{\texttt{vec}}^{\alpha}(v)$, where the frequency of a pair of distance $\lambda$ and degree $\delta$ is contained on index $i = \lambda\cdot(d_{\texttt{max}} + 1) + \delta$. 
Degrees greater than $d_{\texttt{max}}$ are capped to $d_{\texttt{max}}$, with the resulting vector shown in~\autoref{fig:IGELGraphEncoding}:
\begin{align*}
    \textsc{Igel}_{\texttt{vec}}^{\alpha}(v)_i =&   \Big \vert \ldblbrace(\lambda, \delta) \in e^{\alpha}_v\rdblbrace\Big \vert,\\
     \text{for } & \lambda\cdot(d_{\texttt{max}} + 1) + \delta =i.
\end{align*}
$G_1 = (V_1, E_1)$ and $G_2 = (V_1, E_1)$ are $\textsc{Igel}$-equivalent for $\alpha$ if the sorted multi-set containing node representations is the same for $G_1$ and $G_2$:
\begin{gather*}
    G_1 \equiv_{\textsc{Igel}}^{\alpha} G_2 \iff \\
    \ldblbrace e^{\alpha}_{v_1} : \forall v_1 \in V_1 \rdblbrace = \ldblbrace e^{\alpha}_{v_2} : \forall v_2 \in V_2 \rdblbrace.
\end{gather*}

\begin{minipage}[t]{0.49\textwidth}
\begin{algorithm}[H]
    \centering
    \caption{1-WL (Color refinement).}\label{alg:ColorRefinement1WL}
    \begin{algorithmic}[1]
      \Require{$G = (V, E)$}
      \State{$c_v^0 := \texttt{hash}(\ldblbrace d_G(v)\rdblbrace)\ \forall\ v \in V$}
      \Do
        \State{$c_v^{i+1} := \texttt{hash}(\ldblbrace c_u^i:\ \underset{u \neq v}{\forall}\ u \in \mathcal{N}_G^1(v)\rdblbrace) $}
      \DoWhile{$c_v^i \neq c_v^{i-1}$}
      \Ensure{$c_v^i : V \rightarrow \mathbb{N}$}
    \end{algorithmic}
\end{algorithm}
\end{minipage}
\hfill
\begin{minipage}[t]{0.49\textwidth}
\begin{algorithm}[H]
    \centering
    \caption{$\textsc{Igel}$ Encoding.}\label{alg:IGELEncoding}
    \begin{algorithmic}[1]
      \Require{$G = (V, E), \alpha: \mathbb{N}$}
      \State{$e_v^0 := \ldblbrace(0, d_G(v))\rdblbrace\ \forall\ v \in V$}
      \For{$i := 1$; $i \mathrel{+}= 1$ \textbf{until} $i = \alpha$}
        \State{$e_v^i := \bigcup(e_v^{i-1},$}
        \State{$\qquad \ldblbrace(i, d_{\mathcal{E}_G^{\alpha}(v)}(u))$}
        \State{$\qquad~~~\forall u \in \mathcal{N}^{\alpha}_G(v)\  \vert  \ l_G(u, v) = i \rdblbrace)$}
    \EndFor
      \Ensure{$e_v^{\alpha} : V \rightarrow \ldblbrace (\mathbb{N}, \mathbb{N}) \rdblbrace$}
    \end{algorithmic}
\end{algorithm}
\end{minipage}

\textbf{Space complexity.} $\textsc{Igel}$'s worst case space complexity is  $\mathcal{O}(\alpha \cdot n \cdot d_{\texttt{max}})$, conservatively assuming that every node will require $d_{\texttt{max}}$ parameters at every $\alpha$ depth from the center of the ego-network.

\textbf{Time complexity.} For $\textsc{Igel}$, each vertex has $d_{\texttt{max}}$ neighbors where the $\alpha$ iterations imply traversing through geometrically larger ego-networks with $(d_{\texttt{max}})^{\alpha}$ vertices, upper bounded by $m$. Thus $\textsc{Igel}$'s time complexity follows $\mathcal{O}(n \cdot \min(m, (d_{\texttt{max}})^{\alpha}))$, with $\mathcal{O}(n \cdot m)$ when $\alpha \geq \texttt{diam}(G)$, when implemented as BFS, for which we provide further details in~\autoref{app:IGELBFS}.

\section{Theoretical and Experimental Findings}
First, we analyze $\textsc{Igel}$'s expressive power with respect to 1-WL and recent improvements. Second, we measure the impact of $\textsc{Igel}$ as an additional input to enrich existing MP-GNN architectures.

\subsection{Expressivity: Which Graphs are \texorpdfstring{$\textsc{Igel}$}{IGEL}-Distinguishable?}\label{sec:Expressivity}

In this section, we discuss the increased expressivity of $\textsc{Igel}$ with respect to 1-WL, and identify expressivity upper-bounds for graphs that are indistinguishable under \textsc{Matlang} and the 3-WL test. 

\begin{wrapfigure}{r}{0.49\linewidth}
    \vspace{-.6cm}
    \centering
    \includegraphics[width=1.0\linewidth]{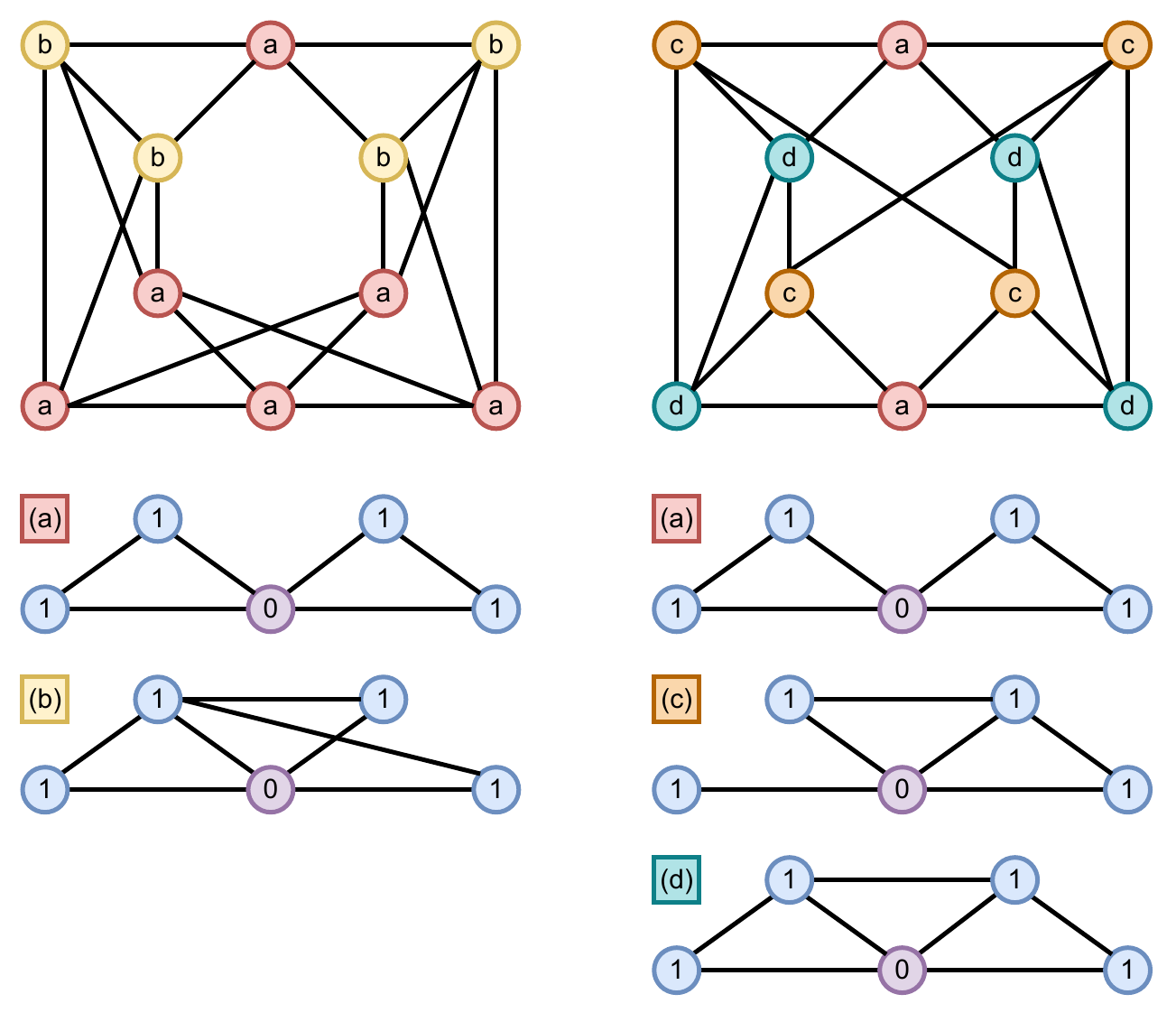}
    \caption{$\textsc{Igel}$ encodings for two Cospectral 4-regular graphs from~\citep{vandam03}. $\textsc{Igel}$ distinguishes 4 kinds of structures within the graphs (associated with every node as \textcolor{red!80!black}{a}, \textcolor{yellow!80!black}{b}, \textcolor{orange!70!black}{c}, and \textcolor{teal!80!black}{d}). The two graphs can be distinguished since the encoded structures and their frequencies do not match.}
    \label{fig:ML2IGELAlpha1}
    \vspace{-30pt}
\end{wrapfigure}

\textbf{--- Relationship to 1-WL.} $\textsc{Igel}$ is more powerful than 1-WL following \autoref{lemma:IGEL1WLPower} (as described and formally shown in~\autoref{app:IGEL1WLExpressivity}) and~\autoref{lemma:IGEL1WLMoreExpressive} (as shown below):

\begin{lemma}\label{lemma:IGEL1WLPower}
    $\textsc{Igel}$ is at least as expressive as 1-WL: $G_1 \not\equiv_{\text{1-WL}} G_2 \Rightarrow G_1 \not\equiv_{\textsc{Igel}}^{\alpha} G_2$ and $G_1 \equiv_{\textsc{Igel}}^{\alpha} G_2 \Rightarrow G_1 \equiv_{\text{1-WL}} G_2$. 
\end{lemma}

\begin{lemma}\label{lemma:IGEL1WLMoreExpressive}
    There exist graphs that $\textsc{Igel}$ can distinguish but that 1-WL cannot distinguish. 
\end{lemma}

\begin{proof}
    \vspace{-9pt}
    For example, any two $d$-regular graphs with equal cardinality are indistinguishable by 1-WL (as shown in~\autoref{app:1WLProof}), but $\textsc{Igel}$ can distinguish some of them. A graph is $d$-regular if all nodes have degree $d$. \autoref{fig:ML2IGELAlpha1} shows two $d$-regular graphs where 1-WL (\autoref{alg:ColorRefinement1WL}) assigns the same color for all nodes, stabilizing after one iteration. In contrast, $\textsc{Igel} (\alpha = 1)$ counts different frequencies for four structures, hence distinguishing between both graphs.
    \vspace{-8pt}
\end{proof}

\textbf{--- Expressivity upper bounds.} We identify an expressivity upper bound for $\textsc{Igel}$, which fails to distinguish \textbf{S}trongly \textbf{R}egular \textbf{G}raphs with equal parameters (\autoref{theo:SRGIgelDistinguishability}, see~\autoref{app:SRGsProof} for details):

\begin{definition}
A $n$-vertex $d$-regular graph is strongly regular---denoted $\texttt{SRG}(n, d, \beta, \gamma)$---if adjacent vertices have $\beta$ vertices in common, and non-adjacent vertices have $\gamma$ vertices in common.
\label{def:SRGs}
\end{definition}

\begin{theorem}\label{theo:SRGIgelDistinguishability}
$\textsc{Igel}$ cannot distinguish \texttt{SRG}s when $n$, $d$, and $\beta$ are the same, and between any value of $\gamma$ (same or otherwise). $\textsc{Igel}$ when $\alpha = 1$ can only distinguish SRGs with different values of $n$, $d$, and $\beta$, while $\textsc{Igel}$ when $\alpha = 2$ can only distinguish SRGs with different values of $n$ and $d$.
\end{theorem}

Our findings show that $\textsc{Igel}$ is a powerful permutation-equivariant representation (see~\autoref{lemma:IGELPermInvariant}), capable of distinguishing 1-WL equivalent graphs as shown in~\autoref{fig:ML2IGELAlpha1}---which as cospectral graphs, are known to be distinguishable in strictly more powerful \textsc{Matlang} sub-languages than 1-WL~\citep{geerts2021}. Additionally, the upper bound on \texttt{SRG}s is a hard ceiling on expressivity since \texttt{SRG}s are known to be indistinguishable by 3-WL~\cite{arvind2020}. $\textsc{Igel}$ shares the experimental upper-bound of expressivity of methods like GNNML3~\citep{balcilar2021breaking}. Furthermore, $\textsc{Igel}$ can provably reach comparable expressivity on \texttt{SRG}s with respect to sub-graph methods implemented within MP-GNN architectures (see~\autoref{app:SRGsProof}, \autoref{subsec:IGELGammaExpressivity}), such as Nested GNNs~\citep{zhang2021} and GNN-AK~\citep{zhao2022}, which are known to be not less powerful than 3-WL, and ESAN when using ego-networks with root-node flags as subgraph sampling policy (EGO+)~\citep{bevilacqua2022equivariant}, which is as powerful as the 3-WL test on \texttt{SRG}s (see~\citep{bevilacqua2022equivariant}, Prop. 3).

\subsection{Experimental Evaluation}\label{sec:Experiments}

We evaluate $\textsc{Igel}^{\alpha}_{\texttt{vec}}(v)$ to produce architecture-agnostic vertex features on five tasks: graph classification, isomorphism detection, graphlet counting, link prediction, and node classification.

\textbf{Experimental Setup.} We introduce $\textsc{Igel}$ on graph classification, isomorphism and graphlet counting, comparing the performance of adding/removing $\textsc{Igel}$ on six GNN architectures following~\cite{balcilar2021breaking}. We also evaluate $\textsc{Igel}$ on link prediction against transductive baselines, and on node classification as additional feature in MLPs without message-passing. \autoref{app:ExperimentalMethodology} describes experimentation details.

\textbf{Notation.} The following formatting denotes significant (as per paired t-tests) \textcolor[HTML]{036400}{\textbf{positive}}, \textcolor[HTML]{9A0000}{\emph{negative}}, and insignificant differences after introducing $\textsc{Igel}$, with the  best results per task / dataset \underline{underlined}.

\begin{table}[ht]
\begin{minipage}[t]{0.54\textwidth}
\caption{Per-model graph classification accuracy metrics on TU data sets. Each cell shows the average accuracy of the model and data set in that row and column, with $\textsc{Igel}$ (left) and without $\textsc{Igel}$ (right).}
\tiny
\begin{center}
\begin{tabular}{@{}lllll@{}}
\toprule
\textbf{Model}   & \multicolumn{1}{c}{\textbf{Enzymes}}                                          & \multicolumn{1}{c}{\textbf{Mutag}}                                            & \multicolumn{1}{c}{\textbf{Proteins}}                                         & \multicolumn{1}{c}{\textbf{PTC}}                                            \\ \midrule
\textbf{MLP}     & {\color[HTML]{036400} \textbf{41.10\textgreater{}26.18}\textsuperscript{$\diamond$}} & {\color[HTML]{036400} \textbf{87.61\textgreater{}84.61}\textsuperscript{$\diamond$}} & 75.43\textasciitilde{}75.01                                                      & {\color[HTML]{036400} \textbf{64.59\textgreater62.79\textsuperscript{$\diamond$}}} \\
\textbf{GCN}     & {\color[HTML]{036400} \textbf{54.48\textgreater{}48.60}\textsuperscript{$\diamond$}} & {\color[HTML]{036400} \textbf{89.61\textgreater{}85.42}\textsuperscript{$\diamond$}} & {\color[HTML]{036400} \textbf{\underline{75.67}\textgreater{}74.50}\textsuperscript{*}}   & 65.76\textasciitilde{}65.21                                                 \\
\textbf{GAT}     & 54.88\textasciitilde{}54.95                                                   & {\color[HTML]{036400} \textbf{90.00\textgreater{}86.14}\textsuperscript{$\diamond$}} & {\color[HTML]{036400} \textbf{73.44\textgreater{}70.51}\textsuperscript{$\diamond$}} & \underline{66.29\textasciitilde{}66.29}                                                 \\
\textbf{GIN}     & {\color[HTML]{036400} \textbf{54.77\textgreater{}53.44}\textsuperscript{*}}   & 89.56\textasciitilde{}88.33                      & {\color[HTML]{036400} \textbf{73.32\textgreater{}72.05}\textsuperscript{$\diamond$}} & 61.44\textasciitilde{}60.21                                                 \\
\textbf{Chebnet} & {\color[HTML]{000000} 61.88\textasciitilde{}62.23}                            & {\color[HTML]{036400} \textbf{91.44\textgreater{}88.33}\textsuperscript{$\diamond$}} & {\color[HTML]{036400} \textbf{74.30\textgreater{}66.94}\textsuperscript{$\diamond$}} & 64.79\textasciitilde{}63.87                                                 \\
\textbf{GNNML3}  & {\color[HTML]{9A0000} \emph{61.42\textless{}\underline{62.79}}\textsuperscript{$\diamond$}}      & {\color[HTML]{036400} \textbf{\underline{92.50}\textgreater{}91.47}\textsuperscript{*}}   & {\color[HTML]{036400} \textbf{75.54\textgreater{}62.32}\textsuperscript{$\diamond$}} & {\color[HTML]{9A0000} \emph{64.26\textless{}66.10}\textsuperscript{$\diamond$}}    \\
\bottomrule
\end{tabular}
\end{center}
{
$\qquad \qquad \qquad$ $\qquad \qquad *\quad p<0.01$, $\qquad \qquad \diamond \quad p<0.0001$
}
\label{tab:GNN-Class}
\end{minipage}
\hfill
\begin{minipage}[t]{0.44\textwidth}
\caption{Mean $\pm$ stddev of best $\textsc{Igel}$ configuration and state-of-the-art results reported on~\citep{nikolentzos2020khop,bouritsas2021,you2021identity,zhang2021,bevilacqua2022equivariant,zhao2022} with \underline{best performing baselines underlined}.}
\tiny
\begin{center}
\begin{tabular}{@{}lccc@{}}
\toprule
\textbf{Model}   & \multicolumn{1}{c}{\textbf{Mutag}}                                            & \multicolumn{1}{c}{\textbf{Proteins}}                                         & \multicolumn{1}{c}{\textbf{PTC}}                                            \\ \midrule
\textbf{\textsc{Igel}} (ours) & $92.5 \pm 1.2$ & $75.7 \pm 0.3$ & $66.3 \pm 1.3$ \\ \midrule
\textbf{$k$-hop~\citep{nikolentzos2020khop}\textsuperscript{$\dagger$}} & {\color[HTML]{9A0000}\textit{$87.9 \pm 1.2$}\textsuperscript{$\diamond$}} & $75.3 \pm 0.4$ & --- \\
\textbf{GSN~\citep{bouritsas2021}\textsuperscript{$\dagger$}} & $92.2 \pm 7.5$ & $76.6 \pm 5.0$ & $68.2 \pm 7.2$ \\
\textbf{NGNN~\citep{zhang2021}\textsuperscript{$\dagger$}} & $87.9 \pm 8.2$ & $74.2 \pm 3.7$ & --- \\
\textbf{ID-GNN~\citep{you2021identity}\textsuperscript{$\dagger$}} & \underline{$93.0 \pm 5.6$} & \color[HTML]{036400}{\underline{$77.9 \pm 2.4$\textsuperscript{*}}} & $62.5 \pm 5.3$ \\
\textbf{GNN-AK~\citep{zhao2022}\textsuperscript{$\dagger$}} & $91.7 \pm 7.0$ & $77.1 \pm 5.7$ & $67.7 \pm 8.8$ \\
\textbf{ESAN~\citep{bevilacqua2022equivariant}\textsuperscript{$\dagger$}} & $91.1 \pm 7.0$ & $76.7 \pm 4.1$ & \underline{$69.2 \pm 6.5$} \\ \bottomrule
\end{tabular}
\end{center}
\begin{center}
  \tiny
  $\dagger$: Results as reported by~\citep{nikolentzos2020khop,bouritsas2021,you2021identity,zhang2021,zhao2022,bevilacqua2022equivariant}.
\end{center}
\label{tab:GNN-Class-Beyond1WL}
\end{minipage}
\vspace{-16pt}
\end{table}

\textbf{--- Graph Classification.}~\autoref{tab:GNN-Class} shows graph classification results on the TU molecule data sets~\cite{Morris2020TUDatasets}. We evaluate differences in mean accuracy between 10 runs with (left) / without (right) $\textsc{Igel}$. We do not tune network hyper-parameters and establish statistical significance through paired t-tests, with  $p<0.01$ (*) and $p<0.0001$ ($\diamond$). Our results show that $\textsc{Igel}$ in the Mutag and Proteins data sets improves the performance of all MP-GNN models, including GNNML3. On the Enzymes and PTC data sets, results are mixed: excluding GNNML3, $\textsc{Igel}$ either significantly improves accuracy (on MLPNet, GCN, and GIN on Enzymes), or does not have a negative impact on performance. 

\autoref{tab:GNN-Class-Beyond1WL} compares $\textsc{Igel}$ results from~\autoref{tab:GNN-Class} with reported results for state-of-the-art 1-WL expressive MP-GNNs. Results are comparable to $\textsc{Igel}$ except where highlighted in color. Overall, when comparing $\textsc{Igel}$ and \underline{best performing baselines}, only differences with ID-GNN on Proteins are statistically significant (using $p$-value threshold $p < 0.01$, where ID-GNN shows $p = 0.009$).

\textbf{--- Isomorphism Detection \& Graphlet Counting.} Adding $\textsc{Igel}$ to the six models in~\autoref{tab:GNN-Class} on the EXP~\citep{abboud2021} isomorphism detection yields significant improvements: all GNN models distinguish all non-isomorphic yet 1-WL equivalent EXP graph pairs with $\textsc{Igel}$ vs. 50\% accuracy without $\textsc{Igel}$ (i.e. random guessing). Additionally, $\textsc{Igel}$ significantly improves GNN graphlet-counting performance on three graphlet types in the RandomGraph data set~\citep{chen2022CanGNNsCountSubstructures}. We provide further details in~\autoref{app:IsoGraphletResults}.

\textbf{--- Link Prediction \& Node Classification.} We test $\textsc{Igel}$ on edge / node level tasks to assess its use as a baseline in non-GNN settings. On a transductive link prediction task, we train DeepWalk~\cite{DBLP:journals/corr/PerozziAS14-deepwalk} style embeddings of $\textsc{Igel}$ encodings rather than node identities on the Facebook and CA-AstroPh graphs~\cite{snapnets}. $\textsc{Igel}$-derived embeddings outperform transductive baselines on link prediction as an edge-level binary classification task, measuring \textcolor[HTML]{036400}{\textbf{\underline{0.976}}} vs. 0.968 (Facebook) and \textcolor[HTML]{036400}{\textbf{\underline{0.984}}} vs. 0.937 (CA-AstroPh) AUC comparing $\textsc{Igel}$ vs. node2vec~\cite{DBLP:journals/corr/GroverL16-node2vec}. On multi-label node classification on PPI~\citep{DBLP:journals/corr/HamiltonYL17-GraphSAGE}, we train an MLP (e.g. no message passing) with node features and $\textsc{Igel}$ encodings. Our MLP shows better micro-F1 (0.850) when $\alpha = 1$ than MP-GNN architectures such as GraphSAGE (0.768, as reported in~\cite{velikovi2017graph-GAT}), but underperforms compared to a 3-layer GAT (\textcolor[HTML]{036400}{\textbf{\underline{0.973}}} micro-F1 from~\cite{velikovi2017graph-GAT}).

\textbf{--- Experimental Summary.} Introducing $\textsc{Igel}$ yields comparable performance to state-of-the-art methods without architectural modifications---including when compared to strong baseline models focused on WL expressivity such as GNNML3, GSN, Nested GNNs, ID-GNN, GNN-AK or ESAN. Furthermore, $\textsc{Igel}$ achieves this at a lower computational cost, in comparison for instance with GNNML3, which requires a $\mathcal{O}(n^3)$ eigen-decomposition step to introduce spectral channels. Finally, $\textsc{Igel}$ can also be used in transductive settings (link prediction) as well as node-level tasks (node classification) and outperform strong transductive baselines or enhance models without message-passing, such as MLPs. As such, we believe $\textsc{Igel}$ is an attractive baseline with a clear relationship to the 1-WL test that improves MP-GNN expressivity without the need for costly architecture search.

\section{Conclusions}

We presented $\textsc{Igel}$, a novel vertex representation algorithm on unattributed graphs allowing MP-GNN architectures to go beyond 1-WL expressivity. We showed that $\textsc{Igel}$ is related and more expressive than the 1-WL test, and formally proved an expressivity upper bound on certain families of Strongly Regular Graphs. Finally, our experimental results indicate that introducing $\textsc{Igel}$ in existing MP-GNN architectures yield comparable performance to state-of-the-art methods, without architectural modifications and at lower computational costs than other approaches.

\if\includeAcknowledgements1

\section*{Author Contributions}
Nurudin Alvarez-Gonzalez: Conceptualization, Methodology, Software, Investigation, Formal analysis, Writing - Original Draft; Andreas Kaltenbrunner: Validation, Supervision, Writing - Review \& Editing; Vicen\c{c} G\'omez: Resources, Validation, Supervision, Writing - Review \& Editing.

\section*{Acknowledgements}

Vicen\c{c} G\'omez has received funding from ``la Caixa'' Foundation (ID 100010434), under the agreement LCF/PR/PR16/51110009. Andreas Kaltenbrunner acknowledges support from Intesa Sanpaolo Innovation Center. The funder had no role in study design, data collection and analysis, decision to publish, or preparation of the manuscript.

\fi

\clearpage

\bibliographystyle{unsrtnat} 
\bibliography{reference}

\begin{thebibliography}{40}
\providecommand{\natexlab}[1]{#1}
\providecommand{\url}[1]{\texttt{#1}}
\expandafter\ifx\csname urlstyle\endcsname\relax
  \providecommand{\doi}[1]{doi: #1}\else
  \providecommand{\doi}{doi: \begingroup \urlstyle{rm}\Url}\fi

\bibitem[Bronstein et~al.(2021)Bronstein, Bruna, Cohen, and Veli{\v
  c}kovi{\'c}]{Bronstein2021GeometricDL}
Michael~M. Bronstein, Joan Bruna, Taco Cohen, and Petar Veli{\v c}kovi{\'c}.
\newblock Geometric deep learning: Grids, groups, graphs, geodesics, and
  gauges.
\newblock \emph{ArXiv}, abs/2104.13478, 2021.

\bibitem[Ying et~al.(2018)Ying, He, Chen, Eksombatchai, Hamilton, and
  Leskovec]{Ying-2018-PinSAGE}
Rex Ying, Ruining He, Kaifeng Chen, Pong Eksombatchai, William~L Hamilton, and
  Jure Leskovec.
\newblock Graph convolutional neural networks for web-scale recommender
  systems.
\newblock In \emph{Proceedings of the 24th ACM International Conference on
  Knowledge Discovery \& Data Mining}, pages 974--983, 2018.

\bibitem[Duvenaud et~al.(2015)Duvenaud, Maclaurin, Iparraguirre, Bombarell,
  Hirzel, Aspuru-Guzik, and Adams]{DuvenaudNIPS2015}
David~K Duvenaud, Dougal Maclaurin, Jorge Iparraguirre, Rafael Bombarell,
  Timothy Hirzel, Alan Aspuru-Guzik, and Ryan~P Adams.
\newblock Convolutional networks on graphs for learning molecular fingerprints.
\newblock In \emph{Advances in Neural Information Processing Systems}, 2015.

\bibitem[Gilmer et~al.(2017)Gilmer, Schoenholz, Riley, Vinyals, and
  Dahl]{gilmer2017neural}
Justin Gilmer, Samuel~S. Schoenholz, Patrick~F. Riley, Oriol Vinyals, and
  George~E. Dahl.
\newblock Neural message passing for quantum chemistry.
\newblock In \emph{Proceedings of the 34th International Conference on Machine
  Learning (ICML)}, volume~70, page 1263–1272, 2017.

\bibitem[Samanta et~al.(2020)Samanta, De, Jana, G\'omez, Chattaraj, Ganguly,
  and Gomez-Rodriguez]{JMLR:v21:19-671}
Bidisha Samanta, Abir De, Gourhari Jana, Vicen\c{c} G\'omez, Pratim Chattaraj,
  Niloy Ganguly, and Manuel Gomez-Rodriguez.
\newblock {NEVAE}: A deep generative model for molecular graphs.
\newblock \emph{Journal of Machine Learning Research}, 21\penalty0
  (114):\penalty0 1--33, 2020.

\bibitem[Battaglia et~al.(2016)Battaglia, Pascanu, Lai, Jimenez~Rezende, and
  Kavukcuoglu]{battaglia_NIPS2016}
Peter Battaglia, Razvan Pascanu, Matthew Lai, Danilo Jimenez~Rezende, and Koray
  Kavukcuoglu.
\newblock Interaction networks for learning about objects, relations and
  physics.
\newblock In \emph{Advances in Neural Information Processing Systems}, 2016.

\bibitem[Xu et~al.(2019)Xu, Hu, Leskovec, and Jegelka]{xu2018how}
Keyulu Xu, Weihua Hu, Jure Leskovec, and Stefanie Jegelka.
\newblock How powerful are graph neural networks?
\newblock In \emph{International Conference on Learning Representations}, 2019.

\bibitem[Morris et~al.(2019)Morris, Ritzert, Fey, Hamilton, Lenssen, Rattan,
  and Grohe]{morris2019}
Christopher Morris, Martin Ritzert, Matthias Fey, William~L. Hamilton, Jan~Eric
  Lenssen, Gaurav Rattan, and Martin Grohe.
\newblock {W}eisfeiler and {L}eman go neural: Higher-order graph neural
  networks.
\newblock \emph{Thirty-Third AAAI Conference on Artificial Intelligence},
  33\penalty0 (01):\penalty0 4602--4609, 2019.

\bibitem[Grohe(2017)]{grohe_2017}
Martin Grohe.
\newblock \emph{Descriptive Complexity, Canonisation, and Definable Graph
  Structure Theory}.
\newblock Lecture Notes in Logic. Cambridge University Press, 2017.

\bibitem[Brijder et~al.(2019)Brijder, Geerts, Bussche, and
  Weerwag]{Brijder2019}
Robert Brijder, Floris Geerts, Jan Van~Den Bussche, and Timmy Weerwag.
\newblock On the expressive power of query languages for matrices.
\newblock \emph{ACM Trans. Database Syst.}, 44\penalty0 (4), 2019.

\bibitem[Barcel\'o et~al.(2020)Barcel\'o, Kostylev, Monet, P\'erez, Reutter,
  and Silva]{Barcelo2020}
Pablo Barcel\'o, Egor~V. Kostylev, Mikael Monet, Jorge P\'erez, Juan Reutter,
  and Juan~Pablo Silva.
\newblock The logical expressiveness of graph neural networks.
\newblock In \emph{International Conference on Learning Representations}, 2020.

\bibitem[Geerts(2021)]{geerts2021}
Floris Geerts.
\newblock On the expressive power of linear algebra on graphs.
\newblock \emph{Theory of Computing Systems}, 65:\penalty0 1--61, 2021.

\bibitem[Morris et~al.(2021)Morris, Lipman, Maron, Rieck, Kriege, Grohe, Fey,
  and Borgwardt]{Morris21a}
Christopher Morris, Yaron Lipman, Haggai Maron, Bastian Rieck, Nils~M Kriege,
  Martin Grohe, Matthias Fey, and Karsten Borgwardt.
\newblock {W}eisfeiler and {L}eman go machine learning: The story so far.
\newblock \emph{arXiv preprint arXiv:2112.09992}, 2021.

\bibitem[You et~al.(2019)You, Ying, and Leskovec]{pmlr-v97-you19b}
Jiaxuan You, Rex Ying, and Jure Leskovec.
\newblock Position-aware graph neural networks.
\newblock In \emph{Proceedings of the 36th International Conference on Machine
  Learning}, volume~97, pages 7134--7143. PMLR, 2019.

\bibitem[Nikolentzos et~al.(2020)Nikolentzos, Dasoulas, and
  Vazirgiannis]{nikolentzos2020khop}
Giannis Nikolentzos, George Dasoulas, and Michalis Vazirgiannis.
\newblock k-hop graph neural networks.
\newblock \emph{Neural Networks}, 130:\penalty0 195--205, 2020.

\bibitem[Bodnar et~al.(2021)Bodnar, Frasca, Otter, Wang, Li\`{o}, Montufar, and
  Bronstein]{bodnar2021}
Cristian Bodnar, Fabrizio Frasca, Nina Otter, Yuguang Wang, Pietro Li\`{o},
  Guido~F Montufar, and Michael Bronstein.
\newblock Weisfeiler and {L}ehman go cellular: {CW} networks.
\newblock In \emph{Advances in Neural Information Processing Systems}, 2021.

\bibitem[Maron et~al.(2019)Maron, Ben-Hamu, Serviansky, and
  Lipman]{HaggaiNIPS2019PPGN}
Haggai Maron, Heli Ben-Hamu, Hadar Serviansky, and Yaron Lipman.
\newblock Provably powerful graph networks.
\newblock In \emph{Advances in Neural Information Processing Systems}, 2019.

\bibitem[Bouritsas et~al.(2021)Bouritsas, Frasca, Zafeiriou, and
  Bronstein]{bouritsas2021}
Giorgos Bouritsas, Fabrizio Frasca, Stefanos Zafeiriou, and Michael~M.
  Bronstein.
\newblock Improving graph neural network expressivity via subgraph isomorphism
  counting, 2021.

\bibitem[Zhang and Li(2021)]{zhang2021}
Muhan Zhang and Pan Li.
\newblock Nested graph neural networks.
\newblock \emph{Advances in Neural Information Processing Systems}, 34, 2021.

\bibitem[Balcilar et~al.(2021)Balcilar, H{\'e}roux, Ga{\"u}z{\`e}re, Vasseur,
  Adam, and Honeine]{balcilar2021breaking}
Muhammet Balcilar, Pierre H{\'e}roux, Benoit Ga{\"u}z{\`e}re, Pascal Vasseur,
  S{\'e}bastien Adam, and Paul Honeine.
\newblock Breaking the limits of message passing graph neural networks.
\newblock In \emph{Proceedings of the 38th International Conference on Machine
  Learning (ICML)}, 2021.

\bibitem[You et~al.(2021)You, Gomes-Selman, Ying, and
  Leskovec]{you2021identity}
Jiaxuan You, Jonathan~M Gomes-Selman, Rex Ying, and Jure Leskovec.
\newblock Identity-aware graph neural networks.
\newblock In \emph{Thirty-Fifth AAAI Conference on Artificial Intelligence},
  volume~35, pages 10737--10745, 2021.

\bibitem[Sandfelder et~al.(2021)Sandfelder, Vijayan, and
  Hamilton]{sandfelder2021ego}
Dylan Sandfelder, Priyesh Vijayan, and William~L Hamilton.
\newblock Ego-gnns: Exploiting ego structures in graph neural networks.
\newblock In \emph{IEEE International Conference on Acoustics, Speech and
  Signal Processing (ICASSP)}, pages 8523--8527. IEEE, 2021.

\bibitem[Zhao et~al.(2022)Zhao, Jin, Akoglu, and Shah]{zhao2022}
Lingxiao Zhao, Wei Jin, Leman Akoglu, and Neil Shah.
\newblock From stars to subgraphs: Uplifting any {GNN} with local structure
  awareness.
\newblock In \emph{International Conference on Learning Representations}, 2022.

\bibitem[Bevilacqua et~al.(2022)Bevilacqua, Frasca, Lim, Srinivasan, Cai,
  Balamurugan, Bronstein, and Maron]{bevilacqua2022equivariant}
Beatrice Bevilacqua, Fabrizio Frasca, Derek Lim, Balasubramaniam Srinivasan,
  Chen Cai, Gopinath Balamurugan, Michael~M. Bronstein, and Haggai Maron.
\newblock Equivariant subgraph aggregation networks.
\newblock In \emph{International Conference on Learning Representations}, 2022.

\bibitem[Li et~al.(2020)Li, Wang, Wang, and Leskovec]{DEA_GNNs}
Pan Li, Yanbang Wang, Hongwei Wang, and Jure Leskovec.
\newblock Distance encoding: Design provably more powerful neural networks for
  graph representation learning.
\newblock In \emph{Advances in Neural Information Processing Systems},
  volume~33, pages 4465--4478, 2020.

\bibitem[Hu et~al.(2020)Hu, Fey, Zitnik, Dong, Ren, Liu, Catasta, and
  Leskovec]{hu2020ogb}
Weihua Hu, Matthias Fey, Marinka Zitnik, Yuxiao Dong, Hongyu Ren, Bowen Liu,
  Michele Catasta, and Jure Leskovec.
\newblock Open graph benchmark: Datasets for machine learning on graphs.
\newblock \emph{arXiv preprint arXiv:2005.00687}, 2020.

\bibitem[Morris et~al.(2020)Morris, Kriege, Bause, Kersting, Mutzel, and
  Neumann]{Morris2020TUDatasets}
Christopher Morris, Nils~M. Kriege, Franka Bause, Kristian Kersting, Petra
  Mutzel, and Marion Neumann.
\newblock Tudataset: A collection of benchmark datasets for learning with
  graphs.
\newblock In \emph{ICML 2020 Workshop on Graph Representation Learning and
  Beyond (GRL+ 2020)}, 2020.

\bibitem[You et~al.(2020)You, Ying, and Leskovec]{you2020design}
Jiaxuan You, Zhitao Ying, and Jure Leskovec.
\newblock Design space for graph neural networks.
\newblock In \emph{Advances in Neural Information Processing Systems}, pages
  17009--17021, 2020.

\bibitem[Dwivedi et~al.(2020)Dwivedi, Joshi, Laurent, Bengio, and
  Bresson]{dwivedi2020benchmarkgnns}
Vijay~Prakash Dwivedi, Chaitanya~K Joshi, Thomas Laurent, Yoshua Bengio, and
  Xavier Bresson.
\newblock Benchmarking graph neural networks.
\newblock \emph{arXiv preprint arXiv:2003.00982}, 2020.

\bibitem[Van~Dam and Haemers(2003)]{vandam03}
Edwin~R Van~Dam and Willem~H Haemers.
\newblock Which graphs are determined by their spectrum?
\newblock \emph{Linear Algebra and its Applications}, 373:\penalty0 241--272,
  2003.

\bibitem[Arvind et~al.(2020)Arvind, Fuhlbrück, Köbler, and
  Verbitsky]{arvind2020}
Vikraman Arvind, Frank Fuhlbrück, Johannes Köbler, and Oleg Verbitsky.
\newblock On {Weisfeiler-Leman} invariance: Subgraph counts and related graph
  properties.
\newblock \emph{Journal of Computer and System Sciences}, 113:\penalty0 42--59,
  2020.

\bibitem[Abboud et~al.(2021)Abboud, Ceylan, Grohe, and Lukasiewicz]{abboud2021}
Ralph Abboud, \.Ismail~\.Ilkan Ceylan, Martin Grohe, and Thomas Lukasiewicz.
\newblock The surprising power of graph neural networks with random node
  initialization.
\newblock In \emph{Proceedings of the Thirtieth International Joint Conference
  on Artificial Intelligence}, pages 2112--2118, 2021.

\bibitem[Chen et~al.(2020)Chen, Chen, Villar, and
  Bruna]{chen2022CanGNNsCountSubstructures}
Zhengdao Chen, Lei Chen, Soledad Villar, and Joan Bruna.
\newblock Can graph neural networks count substructures?
\newblock In \emph{Advances in Neural Information Processing Systems}, 2020.

\bibitem[Perozzi et~al.(2014)Perozzi, Al-Rfou, and
  Skiena]{DBLP:journals/corr/PerozziAS14-deepwalk}
Bryan Perozzi, Rami Al-Rfou, and Steven Skiena.
\newblock Deepwalk: Online learning of social representations.
\newblock In \emph{Proceedings of the 20th ACM SIGKDD International Conference
  on Knowledge Discovery and Data Mining}, pages 701--710, 2014.

\bibitem[Leskovec and Krevl(2014)]{snapnets}
Jure Leskovec and Andrej Krevl.
\newblock {SNAP Datasets}: {Stanford} large network dataset collection, 2014.
\newblock URL \url{http://snap.stanford.edu/data}.

\bibitem[Grover and Leskovec(2016)]{DBLP:journals/corr/GroverL16-node2vec}
Aditya Grover and Jure Leskovec.
\newblock Node2vec: Scalable feature learning for networks.
\newblock In \emph{Proceedings of the 22nd ACM SIGKDD International Conference
  on Knowledge Discovery and Data Mining}, pages 855--864, 2016.

\bibitem[Hamilton et~al.(2017)Hamilton, Ying, and
  Leskovec]{DBLP:journals/corr/HamiltonYL17-GraphSAGE}
Will Hamilton, Zhitao Ying, and Jure Leskovec.
\newblock Inductive representation learning on large graphs.
\newblock In \emph{Advances in Neural Information Processing Systems 30}, 2017.

\bibitem[Veličković et~al.(2018)Veličković, Cucurull, Casanova, Romero,
  Liò, and Bengio]{velikovi2017graph-GAT}
Petar Veličković, Guillem Cucurull, Arantxa Casanova, Adriana Romero, Pietro
  Liò, and Yoshua Bengio.
\newblock Graph attention networks.
\newblock In \emph{International Conference on Learning Representations}, 2018.

\bibitem[Veli{\v{c}}kovi{\'c}(2022)]{velickovic2022message}
Petar Veli{\v{c}}kovi{\'c}.
\newblock Message passing all the way up.
\newblock In \emph{ICLR 2022 Workshop on Geometrical and Topological
  Representation Learning}, 2022.

\bibitem[Anonymous(2023)]{anonymous2023rethinking}
Anonymous.
\newblock Rethinking the expressive power of {GNN}s via graph biconnectivity.
\newblock In \emph{Submitted to The Eleventh International Conference on
  Learning Representations}, 2023.
\newblock under review.

\end{thebibliography}

\clearpage

\appendix
\section{Relation with Previous Works}\label{sec:relwork}
In the past few years, many different approaches have been developed for improving the expressivity of MP-GNNs. 
Here we review the works that are more related to $\textsc{Igel}$.
For a more detailed overview augmented message-passing methods for graph representation learning, see~\citep{velickovic2022message}. 



In \textbf{$k$-hop MP-GNNs ($k$-hop)}~\citep{nikolentzos2020khop} the authors propose to propagate messages beyond immediate vertex neighbors, effectively using ego-network information in the vertex representation.
Their proposed algorithm requires to extract neighborhood sub-graphs and to perform message-passing on each sub-graph, which has an exponential cost on the number of hops $k$ both at pre-processing and at each iteration (epoch). In contrast, $\textsc{Igel}$ only requires a single pre-processing step that can cached once computed.

Distance Encoding GNNs (\textbf{DE-GNN})~\citep{DEA_GNNs} also propose to improve MP-GNN by using extra node features by encoding distances to a subset of $p$ nodes. The features obtained by DE-GNN are similar to IGEL when conditioning the subset to size $p=1$
and using a distance encoding function with $k=\alpha$. However, these features are not strictly equivalent to the $\textsc{Igel}$ features, since within the ego-network the node degrees can be smaller than the actual degrees, and they are more expensive to compute. DE-GNN needs to compute power iterations of the entire adjacency matrix, which is more expensive and does not exploit network sparsity. 

Graph Substructure Networks (\textbf{GSNs})~\citep{bouritsas2021} incorporate hand-crafted topological features by counting local substructures (such as the presence of cliques or cycles). GSNs require expert knowledge on what features are relevant for a given task and depart from the original MP-GNN in their architecture.
We show that $\textsc{Igel}$ reaches comparable performance using a general encoding for ego-networks and without altering the original message-passing mechanism.


\noindent\textbf{GNNML3}~\citep{balcilar2021breaking} proposes a way 
to perform message passing in spectral-domain with a custom frequency profile.
While this approach achieves good performance on graph classification, it requires an expensive preprocessing step for computing the eigendecomposition of the graph Laplacian and $\mathcal{O}(k)$-order tensors to achieve $k$-WL expressiveness, which does not scale to large graphs.

More recently, a series of methods formulate the problem of representing vertices or graphs as aggregations over sub-graphs.
The sub-graph information is pooled or introduced during message-passing at an additional cost that varies depending on each architecture.
Consequently, they require generating the subgraphs (or effectively replicating the nodes of every subgraph of interest) and pay an additional overhead due to the aggregation. These approaches include \textbf{Ego-GNNs}~\cite{sandfelder2021ego},
Nested GNNs (\textbf{NGNNs})~\citep{zhang2021}, GNN-as-Kernel (\textbf{GNN-AK})~\citep{zhao2022}, Identity-aware GNNs (\textbf{ID-GNNs})~\cite{you2021identity}.

\textbf{Ego-GNNs} perform message-passing over the ego-graphs of all the nodes in a graph, and subsequently perform aggregation.
They provide empirical evidence of a superior expressive power than the
classical 1-WL.
\textbf{ID-GNNs} embed each node incorporating identity information in the GNN and apply rounds of heterogeneous message passing; 
\textbf{NGNNs} perform a two-level GNN using rooted sub-graphs and consider a graph as a bag of sub-graphs; \textbf{GNN-AK} uses a very similar idea, but as the authors describe, it sets the number of layers to the number of iterations of 1-WL;
Compared to all these methods $\textsc{Igel}$ only relies on an initial pre-processing step based on distances and degrees without having to run additional message passing iterations.
Despite its simplicity, $\textsc{Igel}$ performs competitively, as we show in Table~\ref{tab:GNN-Class-Beyond1WL}.

Equivariant Subgraph Aggregation Networks (\textbf{ESAN})~\citep{bevilacqua2022equivariant} 
also propose to encode bags of subgraphs and show that such an encoding can lead to a better expressive power. In the case of the ego-networks policy (EGO), \textbf{ESAN} is strongly related with $\textsc{Igel}$. Interestingly, as described in concurrent work~\cite{anonymous2023rethinking}, the implicit encoding of the pairwise distance between nodes, plus the degree information which can be extracted via aggregation are fundamental to provide a theoretical justification of \textbf{ESAN}. In this work, we directly consider distances and degrees in the ego-network, explicitly providing the structural information encoded by more expressive GNN architectures. These similarities may explain why the performance of both methods is comparable, as shown in~\autoref{tab:GNN-Class-Beyond1WL}.

\section{1-WL Expressivity and Regular Graphs}
\label{app:1WLProof}

\autoref{rem:1WL-dRegularProof} shows that 1-WL, as defined in~\autoref{alg:ColorRefinement1WL}, is unable of distinguishing $d$-regular graphs:

\begin{remark}
Let $G_1$ and $G_2$ be two $d$-regular graphs such that $\vert V_1 \vert = \vert V_2 \vert$. Tracing~\autoref{alg:ColorRefinement1WL}, all vertices in $V_1$, $V_2$ share the same initial color due to $d$-regularity: $\forall\ v \in V_1 \bigcup V_2; c_v^0 = \texttt{hash}(\ldblbrace d\rdblbrace)$. After the first color refinement iteration, consider the colorings of $G_1$ and $G_2$:\\
\noindent --- $\forall\ v_1 \in V_1; c_{v_1}^{1} := \texttt{hash}(\ldblbrace c_{u_1}^0:\ \underset{u_1 \neq v_1}{\forall}\ u_1 \in \mathcal{N}_{G_1}^1(v_1)\rdblbrace)$,\\
\noindent --- $\forall\ v_2 \in V_2; c_{v_2}^{1} := \texttt{hash}(\ldblbrace c_{u_2}^0:\ \underset{u_2 \neq v_2}{\forall}\ u_2 \in \mathcal{N}_{G_2}^1(v_2)\rdblbrace)$.\\
Since $\forall\ v_1 \in V_1, v_2 \in V_2; d = \vert \mathcal{N}_{G_1}^1(v_1) \vert = \vert \mathcal{N}_{G_2}^1(v_2)\vert$,  substituting $c_{v_1}^{1}$, $c_{v_2}^{1}$ in the next iteration step yields $\ldblbrace \texttt{hash}(c_{v_1}^{1}): \forall\ v_1 \in V_1 \rdblbrace = \ldblbrace \texttt{hash}(c_{v_2}^{1}): \forall\ v_2 \in V_2 \rdblbrace$. Thus, on any pair of $d$-regular graphs with equal cardinality, 1-WL stabilizes after one iteration produces equal colorings for all nodes on both graphs---regardless of whether $G_1$ and $G_2$ are isomorphic, as~\autoref{fig:ML2IGELAlpha1} shows. \qedhere
\label{rem:1WL-dRegularProof}
\end{remark}

\section{\texorpdfstring{$\textsc{Igel}$}{IGEL} is At Least As Powerful as 1-WL}
\label{app:IGEL1WLExpressivity}
In this section we formally prove~\autoref{lemma:IGEL1WLPower}, i.e. that $\textsc{Igel}$ is at least as expressive as 1-WL.
For this, we consider a variant of 1-WL which removes the hashing step. This modification can only increase the expressive power of 1-WL but makes it possible to directly compare with the encodings generated by $\textsc{Igel}$. 
Intuitively, after $k$ color refinement iterations, 1-WL considers nodes at $k$ hops from each node, which is equivalent to running $\textsc{Igel}$ with $\alpha = k + 1$, so that the ego-networks include the information of all nodes that 1-WL would visit.

\newcounter{currentLemma}
\setcounter{currentLemma}{\value{lemma}}
\setcounter{lemma}{0} 
\begin{lemma}
    $\textsc{Igel}$ is at least as expressive as 1-WL. For two graphs $G_1$, $G_2$ 
    which are distinguished by 1-WL in $k$ iterations ($G_1 \not\equiv_{\text{1-WL}} G_2$) it also holds that
    $G_1 \not\equiv_{\textsc{Igel}}^{\alpha} G_2$ for $\alpha = k + 1$. If $\textsc{Igel}$ does not distinguish two graphs $G_1'$ and $G_2'$, 1-WL also does not distinguish them: $G_1' \equiv_{\textsc{Igel}}^{\alpha} G_2' \Rightarrow G_1' \equiv_{\text{1-WL}} G_2'$.
\end{lemma}
\setcounter{lemma}{\value{currentLemma}}

\begin{proof}[Proof of~\autoref{lemma:IGEL1WLPower}:]
For convenience, let $c^{i+1}_v = \ldblbrace c^{i}_v; c^{i}_u\ \forall\ u \in \mathcal{N}_G^1(v)\ \vert\ u \neq v\rdblbrace$ be a recursive definition of~\autoref{alg:ColorRefinement1WL} where hashing is removed and $c^0_v = \ldblbrace d_G(v) \rdblbrace$. Since the hash is no longer computed, the nested multi-sets contain strictly the same or more information as in the traditional 1-WL algorithm.

For $\textsc{Igel}$ to be less expressive than 1-WL, it must hold that there exist two graphs $G_1 = (V_1, E_1)$ and $G_2 = (V_2, E_2)$ such that $G_1 \not\equiv_{\text{1-WL}} G_2$ while $G_1 \equiv_{\textsc{Igel}}^{\alpha} G_2$.

Let $k$ be the minimum number of color refinement iterations such that $\exists\ v_1 \in V_1$ and $ \forall\ v_2 \in V_2, c_{v_1}^k \neq c_{v_2}^k$. We define an equally or more expressive variant of the 1-WL test 1-WL\textsuperscript{*} where hashing is removed, such that $c_{v_1}^k = \ldblbrace \ldblbrace ... \ldblbrace d_G(v_1) \rdblbrace, \ldblbrace d_G(u) \forall u \in \mathcal{N}_{G_1}^1(v_1) \rdblbrace ... \rdblbrace \rdblbrace$, nested up to depth $k$. To avoid nesting, the multi-set of nested degree multi-sets can be rewritten as the union of degree multi-sets by introducing an indicator variable for the iteration number where a degree is found:
\begin{align*}
    c_{v_1}^k =
              & \bigldblbrace (0, d_G(v_1)) \bigrdblbrace \bigcup \\
              & \bigldblbrace (1, d_G(v_1)); (1, d_G(u))\ \forall\ u \in \mathcal{N}_G^1(v_1) \bigrdblbrace \bigcup \\
              & \bigldblbrace (2, d_G(v_1)); (2, d_G(u))\ \forall\ u \in \mathcal{N}_G^1(v_1); (2, d_G(w))\ \forall\ w \in \mathcal{N}_G^1(u) \bigrdblbrace \bigcup ...
\end{align*}

At each step $i$, we introduce information about nodes up to distance $i$ of $v_1$. Furthermore, by construction, nodes will be visited on every subsequent iteration---i.e. for $c_{v_1}^2$, we will observe $(2, d_G(v_1))$ exactly $d_G(v_1) + 1$ times, as all its $d_G(v_1)$ neighbors $u \in \mathcal{N}_G^1(v)$ encode the degree of $v_1$ in $c_{u}^1$.
The flattened representation provided by 1-WL\textsuperscript{*} is still equally or more expressive than 1-WL, as it removes hashing and keeps track of the iteration at which a degree is found. 

Let $\textsc{Igel-W}$ be a less expressive version of $\textsc{Igel}$ that does not include edges between nodes at $k+1$ hops of the ego-network center. Now consider the case in which $c_{v_1}^k \neq c_{v_2}^k$ from 1-WL\textsuperscript{*}, and let $\alpha = k + 1$ so that $\textsc{Igel-W}$ considers degrees by counting edges found at $k$ to $k + 1$ hops of $v_1$ and $v_2$. Assume that $G_1 \equiv_{\textsc{Igel-W}}^{\alpha} G_2$. By construction, this means that $\ldblbrace e^{\alpha}_{v_1} : \forall\ v_1 \in V_1 \rdblbrace = \ldblbrace e^{\alpha}_{v_2} : \forall\ v_2 \in V_2 \rdblbrace$.
This implies that all degrees and iteration counts match as per the distance indicator variable at which the degrees are found, so $c_{v_1}^k = c_{v_2}^k$ which contradicts the assumption $c_{v_1}^k \neq c_{v_2}^k$ and therefore implies that also $G_1 \equiv_{\text{1-WL}\textsuperscript{*}} G_2$. 
Thus, $G_1 \equiv_{\textsc{Igel-W}}^{\alpha} G_2 \Rightarrow G_1 \equiv_{\text{1-WL}\textsuperscript{*}} G_2$ for $\alpha = k + 1$ and also
$G_1 \not\equiv_{\text{1-WL}\textsuperscript{*}} G_2 \Rightarrow G_1 \not \equiv_{\textsc{Igel-W}}^{\alpha} G_2$. Therefore by extension $\textsc{Igel}$ is at least as expressive as 1-WL.
\end{proof}

\section{\texorpdfstring{$\textsc{Igel}$}{IGEL} is Permutation Equivariant}

\begin{lemma}\label{lemma:IGELPermInvariant}
Given any $v \in V$ for $G = (V, E)$ and given a permuted graph $G' = (V', E')$ of $G$ produced by a permutation of node labels $\pi: V \rightarrow V'$ such that $\forall v \in V \Leftrightarrow \pi(v) \in V'$, $\forall (u, v) \in E \Leftrightarrow (\pi(u), \pi(v)) \in E'$.

The $\textsc{Igel}$ representation is permutation equivariant at the graph level
$$\pi (\ldblbrace e_{v_1}^{\alpha},\dots ,e_{v_n}^{\alpha} \rdblbrace) = \ldblbrace e_{\pi(v_1)}^{\alpha},\dots ,e_{\pi(v_n)}^{\alpha}\rdblbrace.$$
 The $\textsc{Igel}$ representation is permutation invariant at the node level
  $$e_{v}^{\alpha} = e_{\pi(v)}^{\alpha} , \forall v \in G.$$
\end{lemma}

\begin{proof}\label{proof:IGELPermEquivariant}
Note that $e_{v}^{\alpha}$ in~\autoref{alg:IGELEncoding} can be expressed recursively as:
\begin{align*}
    e_{v}^{\alpha} = \bigldblbrace
            \Big(l_{\mathcal{E}^{\alpha}_v}(u, v), d_{\mathcal{E}^{\alpha}_v}(u)\Big) 
        \Big|\ \forall\ u \in \mathcal{N}_G^{\alpha}(v) \bigrdblbrace.
\end{align*}

Since $\textsc{Igel}$ only relies on node distances $l_G(\cdot,\cdot)$ and degree nodes $d_G(\cdot)$, and both $l_G(\cdot,\cdot)$ and $d_G(\cdot)$ are permutation invariant (and the node level) and equivariant (at the graph level) functions, the $\textsc{Igel}$ representation is permutation equivariant at the graph level, and permutation invariant at the node level.
\end{proof}

\section{Proof of Theorem 1}
\label{app:SRGsProof}

In this appendix, we provide proof for~\autoref{theo:SRGIgelDistinguishability}, showing that $\textsc{Igel}$ cannot distinguish certain pairs of \texttt{SRG}s with equal parameters of $n$ (cardinality), $d$ (degree), $\beta$ (shared edges between adjacent nodes), and $\gamma$ (shared edges between non-adjacent nodes). Let $\ldblbrace\cdot\rdblbrace^d$ denote a repeated multi-set with $d$-times the cardinality of the items in the multi-set, and let $e_{G}^{\alpha} = \ldblbrace e^{\alpha}_{v} : \forall\ v \in V \rdblbrace$ be short-hand notation for the $\textsc{Igel}$ encoding of $G$, defined as the sorted multi-set containing $\textsc{Igel}$ encodings of all nodes in $G$.

\begin{lemma}\label{lemma:SRGMaxDiam}
For any $G = \texttt{SRG}(n, d, \beta, \gamma)$, $\texttt{diam}(G) \leq 2$.

Note that by definition of SRGs, $n$ affects cardinality while $d$ and $\beta$ control adjacent vertex connectivity at 1-hop. For $\gamma$, we have to consider two cases: when $\gamma \geq 1$ and when $\gamma = 0$:

--- Let $\gamma \geq 1$: by definition, $\forall\ u, v \in V s.t. (u, v) \notin E, \exists\ w \in V s.t. (u, w) \in E \wedge (v, w) \in E$. Thus, $\forall\ (u, v) \in E, l_G(u, v) = 1$ and $\forall\ (u, v) \notin E, l_G(u, v) = 2$.

--- Let $\gamma = 0$: $\forall\ u, v \in V$, if $(u, v) \notin E$ then $\nexists\ w \in V s.t. (u, w) \in E\ \wedge\ (v, w) \in E$ as $w$ is in common between $u$ and $v$. Then, $\forall\ u, v, w \in V s.t. (u, v) \in E, (u, w) \in E \Leftrightarrow (v, w) \in E$---hence, only nodes and their neighbors can be in common. Thus: $\forall\ u, v \in V s.t. u \neq v, l_G(u, v) = 1$.

Given both scenarios, we can conclude that for any $\gamma \in \mathbb{N}$, $\forall\ u, v \in V, l_G(u, v) \leq 2$ and thus $\texttt{diam}(G) \leq 2$. \null\hfill\qedsymbol
\end{lemma}
~
\begin{lemma}\label{lemma:MaxEncodingAlpha}
For any finite graph $G$, there is a finite range of $\alpha \in \mathbb{N}$ where $\textsc{Igel}$ encodings distinguish between different values of $\alpha$. For values of $\alpha$ larger than the diameter of the graph (that is, $\alpha \geq \texttt{diam}(G)$), it holds that $e_{v}^{\alpha} = e_{v}^{\alpha + 1}$ as $\mathcal{E}_v^{\alpha} = \mathcal{E}_v^{\alpha + 1} = G$. \null\hfill\qedsymbol
\end{lemma}

\begin{proof}\label{proof:IGELSRGExpressivity}
Per~\autoref{lemma:SRGMaxDiam} and~\autoref{lemma:MaxEncodingAlpha}, \texttt{SRG}s have a maximum diameter of two, and $\textsc{Igel}$ encodings are equal for all $\alpha \geq \texttt{diam}(G)$. Thus, given $G = \texttt{SRG}(n, d, \beta, \gamma)$, only $\alpha \in \{1,2\}$ produce different encodings of $G$. It can be shown that $e^{\alpha}_v$ can only distinguish different values of $n$, $d$ and $\beta$, and $\textsc{Igel}_{\texttt{enc}}^{2}$ can only distinguish values of $n$ and $d$:

--- Let $\alpha = 1$: $\forall\ v \in V, \mathcal{E}_v^{1} = (V', E')\ s.t.\ V' = \mathcal{N}_G^{1}(v)$. Since $G$ is $d$-regular, $v$ is the center of $\mathcal{E}_v^{1}$, and has $d$-neighbors. By \texttt{SRG}'s definition, the $d$ neighbors of $v$ have $\beta$ shared neighbors with $v$ each, plus an edge with $v$. Thus, for any \texttt{SRG}s $G_1, G_2$ where $n_1 = n_2$, $d_1 = d_2$, and $\beta_1 = \beta_2$, $e_{G_1}^{1} = e_{G_2}^{1}$ produce equal encodings by expanding $e_v^{1}$ in~\autoref{alg:IGELEncoding}: 
\begin{align*}
    e_{v}^{1} =
    \bigldblbrace \big(0, d\big) \bigrdblbrace ~\mathlarger{\bigcup}~
    {\bigldblbrace \big(1, \beta + 1\big)\bigrdblbrace}^{d}
\end{align*}
--- Let $\alpha = 2$: $\forall\ v \in V, \mathcal{E}_v^{2} = G$ as $\forall\ u \in V, u \in \mathcal{N}_G^{2}(v)$ when $\texttt{diam}(G) \leq 2$. $G$ is $d$-regular, so $\forall\ v \in V, d = d_{\mathcal{E}_v^{2}}(v) =  d_{G}(v)$. Thus, for any \texttt{SRG}s $G_1, G_2$ s.t. $n_1 = n_2$ and $d_1 = d_2$, $e_{G_1}^{2} = e_{G_1}^{2}$, containing $n$ equal $e_v^{2}$ encodings by expanding~\autoref{alg:IGELEncoding}:
\begin{align*}
    e_{v}^{2} =
    \bigldblbrace \big(0, d\big) \bigrdblbrace ~\mathlarger{\bigcup}~         \bigldblbrace \big(1, d\big)\bigrdblbrace^{d} \mathlarger{\bigcup}~
    \bigldblbrace
    \big(2, d\big)\bigrdblbrace^{n-d-1}    
\end{align*}
Thus, $\textsc{Igel}$ cannot distinguish pairs of \texttt{SRG}s when $n$, $d$, and $\beta$ are the same, and between any value of $\gamma$ (equal or different between the pair). $\textsc{Igel}$ when $\alpha = 1$ can only distinguish \texttt{SRG}s with different values of $n$, $d$, and $\beta$, while $\textsc{Igel}$ when $\alpha = 2$ can only distinguish \texttt{SRG}s with different values of $n$ and $d$. \qedhere
\end{proof} 
We note that it is straightforward to extend $\textsc{Igel}$ so that different values of $\gamma$ can be distinguished. We explore one possible extension in~\autoref{subsec:IGELGammaExpressivity}.

\subsection{Improving Expressivity on the \texorpdfstring{$\gamma$}{gamma} Parameter}\label{subsec:IGELGammaExpressivity}

$\textsc{Igel}$ as presented is unable to distinguish between any values of $\gamma$ in $\texttt{SRG}$s. However, $\textsc{Igel}$ can be trivially extended to distinguish between pairs of $\texttt{SRG}$s, bringing parity with methods such as the EGO+ policy in ESAN, NGNNs and GNN-AK. 

Intuitively, $\textsc{Igel}$ is unable to distinguish $\gamma$ because its $(\lambda, \delta)$ tuples are unable to represent relationships between vertices at different distances (i.e. the $\gamma$ parameter). The structural feature definition may be extended to compute the degree between `distance layers' in the sub-graphs, addressing this pitfall. This means modifying $e_v^i$ in~\autoref{alg:IGELEncoding}:
\begin{align*}
    e_v^i = e_v^{i-1} \cup \bigldblbrace \rho(u, v): \forall u \in \mathcal{N}^{\alpha}_G(v)\ \Big\vert \ l_G(u, v) \in \{i, i+1\} \bigrdblbrace
\end{align*}
where:
\begin{align*}
    \rho(u, v) = \Big(
        l_{\mathcal{E}^{\alpha}_v}(u, v), 
        d_{\mathcal{E}^{\alpha}_v}^{0}(u, v),
        d_{\mathcal{E}^{\alpha}_v}^{1}(u, v)
    \Big) 
\end{align*}
and $d_{G}^{p}(u, v)$ generalizes $d_{G}(u)$ to count edges of $u$ at a relative distance $p$ of $v$ in $G = (V, E)$:
\begin{align*}
    d_{G}^{p}(u, v) = \Big \vert (u, w) \in E\ \forall\ w \in V s.t.\ l_{G}(u, w) = l_{G}(u, v) + p\Big \vert.
\end{align*}
It can be shown that this definition of $e_v^i$ is strictly more powerful distinguishing at $\texttt{SRG}$s following an expansion of~\autoref{alg:IGELEncoding} with $\alpha = 2$: 
\begin{align*}
    e_v^{2} = \bigldblbrace \big(0, 0, d\big) \bigrdblbrace\ \mathlarger{\bigcup}~ \bigldblbrace \big(1, \beta, \gamma\big)\bigrdblbrace^{d}\ \mathlarger{\bigcup}~ \bigldblbrace \big(2, d - \gamma, 0\big)\bigrdblbrace^{n-d-1}
\end{align*}
\begin{proof}
For any $G = \texttt{SRG}(n, d, \beta, \gamma)$, $\forall\ v \in V$, $l_{\mathcal{E}^{2}_v}(v, v) = 0$ and there are $d$ edges towards its neighbors---thus the root is encoded as $(0, 0, d)$. Each neighbor is at $l_{\mathcal{E}^{2}_v}(u, v) = 1$, with  $\beta$ edges among each other, and $\gamma$ with vertices not adjacent to $v$---thus $(1, \beta, \gamma)$, where $d = 1 + \beta + \gamma$. By definition, every vertex $w \in V s.t. (u, w) \notin E$ has $\gamma$ neighbors shared with $v$, and $d$ neighbors overall. Per~\autoref{lemma:SRGMaxDiam}, the maximum diameter of $G$ is two, hence $l_{\mathcal{E}^{2}_v}(v, w) = 2$ and for any $w$, the representation is $(2, d - \gamma, 0)$.
\end{proof}

\section{Implementing \texorpdfstring{$\textsc{Igel}$}{IGEL} through Breadth-First Search}\label{app:IGELBFS}
The idea behind the $\textsc{Igel}$ encoding is to represent each vertex $v$ by compactly encoding its corresponding ego-network $\mathcal{E}^{\alpha}_v$ at depth $\alpha$.
The choice of encoding consists of a histogram of vertex degrees at distance $d\leq\alpha$, for each vertex in $\mathcal{E}^{\alpha}_v$.
Essentially, $\textsc{Igel}$ runs a Breadth-First Traversal up to depth $\alpha$, counting the number of times the same degree appears at distance $d\leq \alpha$.

The algorithm shown in~\autoref{alg:IGELEncoding} showcases $\textsc{Igel}$ and its relationship to the 1-WL test. However, in a practical setting, it might be preferable to implement $\textsc{Igel}$ through Breadth-First Search (BFS). In~\autoref{alg:IGELEncodingBFS}, we show one such implementation that fits the time and space complexity described in~\autoref{sec:IGELAlgorithm}:

\begin{algorithm}[H]
    \centering
    \caption{$\textsc{Igel}$ Encoding (BFS).}\label{alg:IGELEncodingBFS}
    \begin{algorithmic}[1]
      \Require{$v \in V, \alpha \in \mathbb{N}$}
      \State $\texttt{toVisit} := [~]$     \Comment{Queue of nodes to visit.}
      \State $\texttt{degrees} := \{~\}$   \Comment{Mapping of nodes to their degrees.}
      \State $\texttt{distances} := \{v: 0\}$ \Comment{Mapping of nodes to their distance to $v$}
      \While {$\texttt{toVisit} \neq \emptyset$}
        \State $u := \texttt{toVisit.dequeue}()$
        \State $\texttt{currentDistance} := \texttt{distances}[u]$
        \State $\texttt{currentDegree} := 0$
        \For {$w \in u\texttt{.neighbors}()$}
            \If {$w \notin \texttt{distances}$}
                \State $\texttt{distances}[w] := \texttt{currentDistance} + 1$ \Comment{$w$ is a new node 1-hop further from $v$.}
            \EndIf
            \If {$\texttt{distances}[w] \leq \alpha$}
                \State $\texttt{currentDegree} := \texttt{currentDegree} + 1$ \Comment{Count edges only within $\alpha$-hops.}
                \If {$w \notin \texttt{degrees}$} \Comment{Enqueue if $w$ has not been visited.}
                    \State $\texttt{toVisit.append}(w)$
                \EndIf
            \EndIf
        \EndFor
        \State $\texttt{degrees}[u] := \texttt{currentDegree}$ \Comment{$u$ is now visited: we know its degree and distance to $v$.}
      \EndWhile
      \State $e_v^{\alpha} = \ldblbrace (\texttt{distances}[u], \texttt{degrees}[u])\ \forall\ u \in \texttt{degrees.keys}() \rdblbrace$ 
      
      \Comment{Produce the multi-set of (distance, degree) pairs for all visited nodes.}
      \Ensure{$e_v^{\alpha} : (\mathbb{N}, \mathbb{N}) \rightarrow \mathbb{N}$}
    \end{algorithmic}
\end{algorithm}

Due to how we structure BFS to count degrees and distances in a single pass, each edge is processed twice---once for each node at end of the edge. It must be noted that when processing every $v \in V$, the time complexity is $\mathcal{O}(n \cdot \min(m, (d_{\texttt{max}})^{\alpha}))$. However, the BFS implementation is also embarrassingly parallel, which means that it can be distributed over $p$ processors with $\mathcal{O}(n \cdot \min(m, (d_{\texttt{max}})^{\alpha}) / p)$ time complexity.

\section{Experimental Settings And Procedures}\label{app:ExperimentalMethodology}

In this section, we provide additional details of our experimental setting. We summarize our datasets and tasks in~\autoref{tab:GraphsOverview}. 

On graph-level tasks, we introduce $\textsc{Igel}$ encodings concatenated to existing vertex features into the best performing model configurations found by~\citep{balcilar2021breaking} without any hyper-parameter tuning (e.g. number of layers, hidden units, choice pooling and activation functions). We evaluate performance differences with and without $\textsc{Igel}$ on each task, data set and model on 10 independent runs, measuring statistical significance of the differences through paired t-tests. 

On vertex and edge-level tasks, we report best performing configurations after hyper-parameter search. Each configuration is evaluated on 5 independent runs. We provide a breakdown of the best performing hyper-parameters in the section below.

\subsection{Hyper-parameters and Experiment Details}

\textbf{Graph Level Experiments}

We reproduce the benchmark of~\citep{balcilar2021breaking} without modifying model hyper-parameters for the tasks of Graph Classification, Graph Isomorphism Detection, and Graphlet Counting. For classification tasks, the 6 models in~\autoref{tab:GNN-Class-Beyond1WL} are trained on binary / categorical cross-entropy objectives depending on the task. For Graph Isomorphism Detection, we train GNNs as binary classification models on the binary classification task on EXP~\citep{abboud2021}, and identify isomorphisms by counting the number of graph pairs for which randomly initialized MP-GNN models produce equivalent outputs on Graph8c\footnote{Simple 8 vertices graphs from: \url{http://users.cecs.anu.edu.au/~bdm/data/graphs.html}}\footnote{That is, models are not trained but simply initialized, following the approach of~\citep{balcilar2021breaking}.}. For the graphlet counting regression task on the RandomGraph data set~\citep{chen2022CanGNNsCountSubstructures}, we train models to minimize Mean Squared Error (MSE) on the normalized graphlet counts\footnote{Counts are normalized by the standard deviation counts across the data set for MSE values to be consistent across graphlet types, in alignment with~\citep{balcilar2021breaking}.} for five types of graphlets. 

On all tasks, we experiment with $\alpha \in \{1, 2\}$ and optionally introduce a preliminary linear transformation layer to reduce the dimensionality of $\textsc{Igel}$ encodings. For every setup, we execute the same configuration 10 times with different seeds and compare runs introducing $\textsc{Igel}$ or not by measuring whether differences on the target metric (e.g. accuracy or MSE) are statistically significant as shown in~\autoref{tab:GNN-Class} and~\autoref{tab:GNN-Class-Beyond1WL}. In~\autoref{tab:GraphLevelAlphas}, we provide the value of $\alpha$ that was used in our experimental results. Our results show that the choice of $\alpha$ depends on both the task and model type. We believe these results may be applicable to subgraph-based MP-GNNs, and will explore how different settings, graph sizes, and downstream models interact with $\alpha$ in future work.

\begin{table}[!t]
\centering
\caption{Values of $\alpha$ used when introducing $\textsc{Igel}$ in the best reported configuration for graphlet counting and graph classification tasks. The table is broken down by graphlet types (upper section) and graph classification tasks on the TU Datasets (bottom section).}
\begin{tabular}{@{}crrrrrrr@{}}
\toprule
\multicolumn{1}{l}{}     & \multicolumn{1}{c}{\textbf{Chebnet}} & \multicolumn{1}{c}{\textbf{GAT}} & \multicolumn{1}{c}{\textbf{GCN}} & \multicolumn{1}{c}{\textbf{GIN}} & \multicolumn{1}{c}{\textbf{GNNML3}} & \multicolumn{1}{c}{\textbf{Linear}} & \multicolumn{1}{c}{\textbf{MLP}} \\ \midrule
\textbf{Star}            & 2                                    & 1                                & 2                                & 1                                & 1                                   & 2                                   & 1                                \\
\textbf{Tailed Triangle} & 1                                    & 1                                & 1                                & 1                                & 2                                   & 1                                   & 1                                \\
\textbf{Triangle}        & 1                                    & 1                                & 1                                & 1                                & 1                                   & 1                                   & 1                                \\
\textbf{4-Cycle}         & 2                                    & 1                                & 1                                & 1                                & 1                                   & 1                                   & 1                                \\
\textbf{Custom Graphlet} & 2                                    & 1                                & 1                                & 1                                & 2                                   & 2                                   & 2                                \\ \midrule
\textbf{Enzymes}         & 1                                    & 2                                & 2                                & 1                                & 2                                   & 2                                   & 2                             \\
\textbf{Mutag}           & 1                                    & 1                                & 1                                & 1                                & 1                                   & 1                                   & 2                                \\
\textbf{Proteins}        & 2                                    & 2                                & 2                                & 1                                & 2                                   & 1                                   & 1                                \\
\textbf{PTC}             & 1                                    & 1                                & 2                                & 1                                & 1                                   & 2                                   & 2                                \\ \bottomrule
\end{tabular}
\label{tab:GraphLevelAlphas}
\end{table}

\emph{Reproducibility}-- We provide an additional repository with our changes to the original benchmark, including our modelling scripts, metadata, and experimental results\footnote{\url{https://github.com/nur-ag/gnn-matlang}}.

\textbf{Vertex and Edge-level Experiments}

In this section we break down the best performing hyper-parameters on the Edge (link prediction) and Vertex-level (node classification) experiments.

\emph{Link Prediction}--
The best performing hyperparameter configuration on the Facebook graph including $\alpha = 2$, learning $t = 256$ component vectors with $e = 10$ walks per node, each of length $s = 150$ and $p = 8$ negative samples per positive for the self-supervised negative sampling. Respectively on the arXiv citation graph, we find the best configuration at $\alpha = 2$, $t = 256$, $e = 2$, $s = 100$ and $p = 9$. 

\emph{Node Classification}--
We analyze both encoding distances $\alpha \in \{1, 2\}$. Other $\textsc{Igel}$ hyper-parameters are fixed after a small greedy search based on the best configurations in the link prediction experiments. For the MLP model, we perform greedy architecture search, including number of hidden units, activation functions and depth. Our results show scores averaged over five different seeded runs with the same configuration obtained from hyperparameter search. 

The best performing hyperparameter configuration on the node classification is found with $\alpha = 2$ on $t = 256$ length embedding vectors, concatenated with node features as the input layer for 1000 epochs in a 3-layer MLP using ELU activations with a learning rate of 0.005. Additionally, we apply 100 epoch patience for early stopping, monitoring the F1-score on the validation set.

\emph{Reproducibility}--
We provide a replication folder in the code repository for the exact configurations used to run the experiments\footnote{\url{https://github.com/nur-ag/IGEL}}.

\section{Extended Results on Isomorphism Detection and Graphlet Counting}
\label{app:IsoGraphletResults}

In this section we summarize additional results on isomorphism detection and graphlet counting.

\subsection{Isomorphism Detection}

We provide a detailed breakdown of isomorphism detection performance after introducing $\textsc{Igel}$ in~\autoref{tab:GNN-Similarity}, complimenting our summary on~\autoref{sec:Experiments}.

\begin{wraptable}{r}{0.44\linewidth}
\vspace{-16pt}
\caption{Graph isomorphism detection results. The $\textsc{Igel}$ column denotes whether $\textsc{Igel}$ is used or not in the configuration. For Graph8c, we describe graph pairs erroneously detected as isomorphic. For EXP classify, we show the accuracy of distinguishing non-isomorphic graphs in a binary classification task.}
\tiny
\centering
\begin{tabular}{@{}lcrr@{}}
\toprule
\textbf{Model}                     & \multicolumn{1}{l}{\textbf{+ $\textsc{Igel}$}} & \multicolumn{1}{c}{\textbf{Graph8c}} & \multicolumn{1}{c}{\textbf{EXP Classify}} \\
& &  \multicolumn{1}{c}{\textbf{(\#Errors)}} &  \multicolumn{1}{c}{\textbf{(Accuracy)}} 
\\ \midrule
                                   & No                                  & 6.242M                               & 50\%                                      \\
\multirow{-2}{*}{\textbf{Linear}}  & \textbf{Yes}                        & {\color[HTML]{036400} \textbf{1571}} & {\color[HTML]{036400} \textbf{97.25\%}}   \\ \midrule
                                   & No                                  & 293K                                 & 50\%                                      \\
\multirow{-2}{*}{\textbf{MLP}}     & \textbf{Yes}                        & {\color[HTML]{036400} \textbf{1487}} & {\color[HTML]{036400} \textbf{100\%}}     \\ \midrule
                                   & No                                  & 4196                                 & 50\%                                      \\
\multirow{-2}{*}{\textbf{GCN}}     & \textbf{Yes}                        & {\color[HTML]{036400} \textbf{5}}    & {\color[HTML]{036400} \textbf{100\%}}     \\ \midrule
                                   & No                                  & 1827                                 & 50\%                                      \\
\multirow{-2}{*}{\textbf{GAT}}     & \textbf{Yes}                        & {\color[HTML]{036400} \textbf{5}}    & {\color[HTML]{036400} \textbf{100\%}}     \\ \midrule
                                   & No                                  & 571                                  & 50\%                                      \\
\multirow{-2}{*}{\textbf{GIN}}     & \textbf{Yes}                        & {\color[HTML]{036400} \textbf{5}}    & {\color[HTML]{036400} \textbf{100\%}}     \\ \midrule
                                   & No                                  & 44                                   & 50\%                                      \\
\multirow{-2}{*}{\textbf{Chebnet}} & \textbf{Yes}                        & {\color[HTML]{036400} \textbf{1}}    & {\color[HTML]{036400} \textbf{100\%}}     \\ \midrule
                                   & No                                  & 0                                    & 100\%                                     \\
\multirow{-2}{*}{\textbf{GNNML3}}  & \textbf{Yes}                        & 0                                    & 100\%                                     \\ \bottomrule
\end{tabular}
\label{tab:GNN-Similarity}
\end{wraptable}

\textbf{--- Graph8c.} On the Graph8c dataset, introducing $\textsc{Igel}$ significantly reduces the amount of graph pairs erroneously identified as isomorphic for all MP-GNN models, as shown in~\autoref{tab:GNN-Similarity}. Furthermore, $\textsc{Igel}$ allows a linear baseline employing a sum readout function over input feature vectors, then projecting onto a 10-component space, to identify all but 1571 non-isomorphic pairs compared to the erroneous pairs GCNs (4196 errors) or GATs (1827 errors) can identify without $\textsc{Igel}$. Additionally, we find that all Graph8c graphs can be distinguished if the $\textsc{Igel}$ encodings for $\alpha = 1$ and $\alpha = 2$ are concatenated. We do not explore the expressivity of combinations of $\alpha$ in this work, but hypothesize that concatenated encodings of $\alpha$ may be more expressive.

\textbf{--- Empirical Results on Strongly Regular Graphs.} We also evaluate $\textsc{Igel}$ on SR25\footnote{$\texttt{SRG}(25,12,5,6)$ graphs from: \url{http://users.cecs.anu.edu.au/~bdm/data/graphs.html}}, which contains 15 Strongly Regular graphs with 25 vertices, known to be indistinguishable by 3-WL. With SR25, we validate~\autoref{theo:SRGIgelDistinguishability}. \citep{balcilar2021breaking} showed that no models in our benchmark distinguish any of the 105 non-isomorphic graph pairs in SR25. As expected from~\autoref{theo:SRGIgelDistinguishability}, $\textsc{Igel}$ does not improve distinguishability. 

\subsection{Graphlet Counting}

We evaluate $\textsc{Igel}$ on a (regression) graphlet\footnote{3-stars, triangles, tailed triangles and 4-cycles, plus a custom 1-WL graphlet proposed in~\citep{balcilar2021breaking}} counting task. We minimize Mean Squared Error (MSE) on normalized graphlet counts\footnote{Counts are stddev-normalized so that MSE values are comparable across graphlet types, following~\citep{balcilar2021breaking}.}. \autoref{tab:GNN-Counting} shows the results of introducing $\textsc{Igel}$ in 5 graphlet counting tasks on the RandomGraph data set~\citep{chen2022CanGNNsCountSubstructures}. Stat sig. differences ($p<0.0001$) shown in \textcolor[HTML]{036400}{\textbf{bold green}}, with best (lowest MSE) per-graphlet results \underline{underlined}.

\begin{wraptable}{r}{0.6\linewidth}
\vspace{-.5cm}
\caption{Graphlet counting results. Cells contain mean test set MSE error (lower is better), stat. sig \textbf{\color[HTML]{036400}highlighted}.}
\tiny
\centering
\begin{tabular}{@{}ccrrrrr@{}}
\toprule
\textbf{Model}                     & \textbf{+ $\textsc{Igel}$}                     & \multicolumn{1}{c}{\textbf{Star}}        & \multicolumn{1}{c}{\textbf{Triangle}}    & \multicolumn{1}{c}{\textbf{Tailed Tri.}} & \multicolumn{1}{c}{\textbf{4-Cycle}}     & \multicolumn{1}{c}{\textbf{Custom}}      \\ \midrule
                                   & No                                  & 1.60E-01                                 & 3.41E-01                                 & 2.82E-01                                 & 2.03E-01                                 & 5.11E-01                                 \\
\multirow{-2}{*}{\textbf{Linear}}  & {\color[HTML]{036400} \textbf{Yes}} & {\color[HTML]{036400} \textbf{4.23E-03}} & {\color[HTML]{036400} \textbf{4.38E-03}} & {\color[HTML]{036400} \textbf{1.85E-02}} & {                              1.36E-01} & {\color[HTML]{036400} \textbf{5.25E-02}} \\ \midrule
                                   & No                                  & \underline{2.66E-06}                                 & 2.56E-01                                 & 1.60E-01                                 & 1.18E-01                                 & 4.54E-01                                 \\
\multirow{-2}{*}{\textbf{MLP}}     & {\color[HTML]{036400} \textbf{Yes}} & 8.31E-05                                 & {\color[HTML]{036400} \textbf{\underline{5.69E-05}}} & {\color[HTML]{036400} \textbf{\underline{5.57E-05}}} & {\color[HTML]{036400} \textbf{7.64E-02}} & {\color[HTML]{036400} \textbf{\underline{2.34E-04}}} \\ \midrule
                                   & No                                  & 4.72E-04                                 & 2.42E-01                                 & 1.35E-01                                 & 1.11E-01                                 & 1.54E-03                                 \\
\multirow{-2}{*}{\textbf{GCN}}     & {\color[HTML]{036400} \textbf{Yes}} & 8.26E-04                                 & {\color[HTML]{036400} \textbf{1.25E-03}} & {\color[HTML]{036400} \textbf{4.15E-03}} & {\color[HTML]{036400} \textbf{7.32E-02}} & {\color[HTML]{036400} \textbf{1.17E-03}} \\ \midrule
                                   & No                                  & 4.15E-04                                 & 2.35E-01                                 & 1.28E-01                                 & 1.11E-01                                 & 2.85E-03                                 \\
\multirow{-2}{*}{\textbf{GAT}}     & {\color[HTML]{036400} \textbf{Yes}} & 4.52E-04                                 & {\color[HTML]{036400} \textbf{6.22E-04}} & {\color[HTML]{036400} \textbf{7.77E-04}} & {\color[HTML]{036400} \textbf{7.33E-02}} & {\color[HTML]{036400} \textbf{6.66E-04}} \\ \midrule
                                   & No                                  & 3.17E-04                                 & 2.26E-01                                 & 1.22E-01                                 & 1.11E-01                                 & 2.69E-03                                 \\
\multirow{-2}{*}{\textbf{GIN}}     & {\color[HTML]{036400} \textbf{Yes}} & 6.09E-04                                 & {\color[HTML]{036400} \textbf{1.03E-03}} & {\color[HTML]{036400} \textbf{2.72E-03}} & {\color[HTML]{036400} \textbf{6.98E-02}} & {\color[HTML]{036400} \textbf{2.18E-03}} \\ \midrule
                                   & No                                  & 5.79E-04                                 & 1.71E-01                                 & 1.12E-01                                 & 8.95E-02                                 & 2.06E-03                                 \\
\multirow{-2}{*}{\textbf{Chebnet}} & {\color[HTML]{036400} \textbf{Yes}} & 3.81E-03                                 & {\color[HTML]{036400} \textbf{7.88E-04}} & {\color[HTML]{036400} \textbf{2.10E-03}} & {                              7.90E-02} & {\color[HTML]{036400} \textbf{2.05E-03}} \\ \midrule
                                   & No                                  & 8.90E-05                                 & 2.36E-04                                 & 2.91E-04                                 & \underline{6.82E-04}                                 & 9.86E-04                                 \\
\multirow{-2}{*}{\textbf{GNNML3}}  & Yes                                 & 9.29E-04                                 & 2.19E-04                                 & 4.23E-04                                 & 6.98E-04                                 & 4.17E-04                                 \\ \bottomrule
\end{tabular}
\vspace{-16pt}
\label{tab:GNN-Counting}
\end{wraptable}

Introducing $\textsc{Igel}$ improves counting performance on triangles, tailed triangles and the custom 1-WL graphlets proposed by~\citep{balcilar2021breaking}. Star graphlets can be identified by all baselines, and $\textsc{Igel}$ only produces statistically significant improvements for the Linear baseline.

Notably, the Linear baseline plus $\textsc{Igel}$ outperforms MP-GNNs without $\textsc{Igel}$ for star, triangle, tailed triangle and custom 1-WL graphlets. By introducing $\textsc{Igel}$ on the MLP baseline, it outperforms all other models including GNNML3 on the triangle, tailed-triangle and custom 1-WL graphlets. 

Since Linear and MLP baselines do not use message passing, we believe raw $\textsc{Igel}$ encodings may be sufficient to identify certain graph structures even with simple linear models. For all graphlets except 4-cycles, introducing $\textsc{Igel}$ yields performance similar to GNNML3 at lower pre-processing and model training/inference costs, as $\textsc{Igel}$ obviates the need for costly eigen-decomposition and can be used in simple models only performing graph-level readouts without message passing.

\begin{sidewaystable}[!ht]
\centering
\caption{Overview of the graphs used in the experiments. We show the average number of vertices (Avg. $n$), edges (Avg. $m$), number of graphs, target task, output shape, and splits (when applicable).}
\begin{tabular}{@{}crrrccc@{}}
\toprule
\multicolumn{1}{l}{}                                              & \multicolumn{1}{c}{\textbf{Avg. $n$}} & \multicolumn{1}{c}{\textbf{Avg. $m$}} & \multicolumn{1}{c}{\textbf{\begin{tabular}[c]{@{}c@{}}Num. \\ Graphs\end{tabular}}} & \textbf{Task}                                                                                 & \textbf{Output Shape}                                                               & \textbf{\begin{tabular}[c]{@{}c@{}}Splits\\ (Train / Valid / Test)\end{tabular}}        \\ \midrule
\textbf{Enzymes}                                                  & 32.63                                 & 62.14                                 & 600                                                                                 & \begin{tabular}[c]{@{}c@{}}Multi-class\\ Graph Class.\end{tabular}                            & \begin{tabular}[c]{@{}c@{}}6 (multi-class\\ probabilities)\end{tabular}             & \begin{tabular}[c]{@{}c@{}}9-fold / 1 fold\\ (Graphs, Train / Eval)\end{tabular}        \\
\textbf{Mutag}                                                    & 17.93                                 & 39.58                                 & 188                                                                                 & \begin{tabular}[c]{@{}c@{}}Binary\\ Graph Class.\end{tabular}                                 & \begin{tabular}[c]{@{}c@{}}2 (binary class\\ probabilities)\end{tabular}            & \begin{tabular}[c]{@{}c@{}}9-fold / 1 fold\\ (Graphs, Train / Eval)\end{tabular}        \\
\textbf{Proteins}                                                 & 39.06                                 & 72.82                                 & 1113                                                                                & \begin{tabular}[c]{@{}c@{}}Binary\\ Graph Class.\end{tabular}                                 & \begin{tabular}[c]{@{}c@{}}2 (binary class\\ probabilities)\end{tabular}            & \begin{tabular}[c]{@{}c@{}}9-fold / 1 fold\\ (Graphs, Train / Eval)\end{tabular}        \\
\textbf{PTC}                                                      & 25.55                                 & 51.92                                 & 344                                                                                 & \begin{tabular}[c]{@{}c@{}}Binary\\ Graph Class.\end{tabular}                                 & \begin{tabular}[c]{@{}c@{}}2 (binary class\\ probabilities)\end{tabular}            & \begin{tabular}[c]{@{}c@{}}9-fold / 1 fold\\ (Graphs, Train / Eval)\end{tabular}        \\ \midrule
\textbf{Graph8c}                                                  & 8.0                                   & 28.82                                 & 11117                                                                               & \begin{tabular}[c]{@{}c@{}}Non-isomorphism\\ Detection\end{tabular}                           & N/A                                                                                 & N/A                                                                                     \\
\textbf{EXP Classify}                                             & 44.44                                 & 111.21                                & 600                                                                                 & \begin{tabular}[c]{@{}c@{}}Binary Class.\\ (pairwise graph\\ distinguishability)\end{tabular} & \begin{tabular}[c]{@{}c@{}}1 (non-isomorphic\\ graph pair probability)\end{tabular} & \begin{tabular}[c]{@{}c@{}}Graph pairs\\ 400 / 100 / 100\end{tabular}                   \\
\textbf{SR25}                                                     & 25                                    & 300                                   & 15                                                                                  & \begin{tabular}[c]{@{}c@{}}Non-isomorphism\\ Detection\end{tabular}                           & N/A                                                                                 & N/A                                                                                     \\ \midrule
\textbf{RandomGraph}                                              & 18.8                                  & 62.67                                 & 5000                                                                                & \begin{tabular}[c]{@{}c@{}}Regression\\ (Graphlet Counting)\end{tabular}                      & 1 (graphlet counts)                                                                 & \begin{tabular}[c]{@{}c@{}}Graphs\\ 1500 / 1000 / 2500\end{tabular}                     \\ \midrule
\textbf{\begin{tabular}[c]{@{}c@{}}ArXiv\\ ASTRO-PH\end{tabular}} & 18722                                 & 198110                                & 1                                                                                   & \begin{tabular}[c]{@{}c@{}}Binary Class.\\ (Link Prediction)\end{tabular}                     & 1 (edge probability)                                                                & \begin{tabular}[c]{@{}c@{}}Randomly sampled edges\\ 50\% train / 50\% test\end{tabular} \\
\textbf{Facebook}                                                 & 4039                                  & 88234                                 & 1                                                                                   & \begin{tabular}[c]{@{}c@{}}Binary Class.\\ (Link Prediction)\end{tabular}                     & 1 (edge probability)                                                                & \begin{tabular}[c]{@{}c@{}}Randomly sampled edges\\ 50\% train / 50\% test\end{tabular} \\ \midrule
\textbf{PPI}                                                      & 2373                                  & 68342.4                               & 24                                                                                  & \begin{tabular}[c]{@{}c@{}}Multi-label\\ Vertex Class.\end{tabular}                           & \begin{tabular}[c]{@{}c@{}}121 (binary\\ class probabilities)\end{tabular}          & \begin{tabular}[c]{@{}c@{}}Graphs\\ 20 / 2 / 2\end{tabular}                             \\ \bottomrule
\end{tabular}
\label{tab:GraphsOverview}
\end{sidewaystable}

\end{document}